\documentclass{article}
\usepackage[accepted]{icml2021}

\usepackage{amsfonts,  amsmath, amssymb, amsthm}
\usepackage{float}
\usepackage{xfrac}
\usepackage{graphicx, color, colortbl}
\usepackage{dsfont}
\usepackage{bbm}
\usepackage{import}
\usepackage{mathtools}
\usepackage{bm}
\usepackage[colorlinks=true, allcolors=blue]{hyperref}
\usepackage{xcolor}
\definecolor{dmorange500}{HTML}{FF5F19}
\definecolor{dmblue300}{HTML}{2267EB}
\definecolor{dmred300}{HTML}{FF617B}
\hypersetup{
	colorlinks,
	citecolor=dmblue300,
	linkcolor=dmred300,
	urlcolor=dmorange500}
\usepackage{xspace}
\usepackage{enumitem}
\setitemize{noitemsep,topsep=0pt,parsep=0pt,partopsep=0pt}

\usepackage{sty/notations}

\icmltitlerunning{Problem Dependent Thresholding Bandit Problems}

\begin{document}

\twocolumn[
\icmltitle{Problem Dependent View on Structured Thresholding Bandit Problems}

\icmlsetsymbol{equal}{*}

\begin{icmlauthorlist}
\icmlauthor{James Cheshire}{ovgu}
\icmlauthor{Pierre M\'enard}{ovgu}
\icmlauthor{Alexandra Carpentier}{ovgu}
\end{icmlauthorlist}
\icmlaffiliation{ovgu}{Otto von Guericke
University Magdeburg}

\icmlcorrespondingauthor{James Cheschire}{james.cheschire@ovgu.de}
% \icmlcorrespondingauthor{Pierre M\'enard}{pierre.menard@ovgu.de}

\icmlkeywords{multi-armed bandits, combinatorial pure exploration, thresholding bandits, problem dependent, binary search, ICML}
\vskip 0.3in
]

\printAffiliationsAndNotice{}  

\begin{abstract}%

 We investigate the \textit{problem dependent regime} in the stochastic \emph{Thresholding Bandit problem} (\tbp) under several \emph{shape constraints}. In the \tbp the objective of the learner is to output, at the end of a sequential game, the set of arms whose means are above a given threshold. The vanilla, unstructured, case is already well studied in the literature. Taking $K$ as the number of arms, we consider the case where (i) the sequence of arm's means $(\mu_k)_{k=1}^K$ is monotonically increasing (\textit{MTBP}) and (ii) the case where $(\mu_k)_{k=1}^K$ is concave (\textit{CTBP}). We consider both cases in the \emph{problem dependent} regime and study the probability of error - i.e.~the probability to mis-classify at least one arm. In the fixed budget setting, we provide upper and lower bounds for the probability of error in both the concave and monotone settings, as well as associated algorithms. In both settings the bounds match in the \emph{problem dependent} regime up to universal constants in the exponential.
\end{abstract}

% !TEX root = ../main.tex
\section{Introduction}
\label{sec:introduction}

Stochastic multi-armed bandit problems model situations in which a learner faces multiple unknown probability distributions, or ``arms'', and has to sequentially sample these arms. 

In this paper, we focus on the Thresholding Bandit Problem (\tbp), a \textit{Combinatorial Pure Exploration (CPE)} bandit setting introduced by \citet{chen2014combinatorial}. The learner is presented with $[K] = \{1,\ldots,K\}$ arms, each following an unknown distribution $\nu_k$ with unknown mean $\mu_k$.  We focus on the \emph{fixed budget} variant of this problem. Given a budget $T>0$, the learner samples the arms sequentially for a total of $T$ times and then aims at predicting the set of arms whose mean is above a known threshold $\tau\in\R$. We will measure the learner's performance by the \emph{probability of error} - i.e. the probability that the learner mis-classifies at least one arm - and consider therefore the \emph{problem dependent regime}. %This is in contrast to the problem independent regime where it is classical to take the expected simple regret as a measure of performance.
%We aim to derive matching upper and lower bounds for probability of error.

The focus of this paper is on \emph{structured, shape constrained \tbp}. More precisely, we study the influence of some classical \textit{structures, in the form of a shape constraint} on the \textit{sequence of means of the arms}, on the \tbp problem. That is, we study how classical shape constraints influence the probability of error. A related study was performed by \citet{cheshire2020} for the problem independent (overall worst-case) regime, and we aim at extending this study to the \textit{problem dependent regime}. We will aim at finding the problem dependent quantities that have an impact on the optimal probability of error, and at providing matching upper and lower bounds. 

We will discuss three structured \tbp \hspace{-4pt}s in this paper; among those, we recall existing results of one, and provide results for two. Here is a short overview.
%from the perspective of both upper and lower bounds. We will consider four shape constraints. 

% In this paper, we will be interested only in the \emph{problem dependent} setting, and aim at finding the problem dependent quantities that have an impact on the probability of error - for the \emph{problem independent} setting see \cite{cheshire2020}. 
\vspace{-0.3cm}
\paragraph{Vanilla, unstructured case \tbp}
The vanilla, unstructured case is the simplest \tbp where we only assume that the distributions of the arms are sub-Gaussian - also related to the TOP-M\footnote{In the TOP-M setting, the objective of the learner is to output the $M$ arms with highest means. A popular version of it it is the TOP-1 or "best arm identification" problem where the aim is to find the arm that realises the maximum.} setting. The \tbp is already well studied in the literature - both in a fixed budget and in a fixed confidence context - and we only introduce it here to provide a benchmark for later structured problems. We recall here results in the problem dependent, fixed budget, setting, which is most relevant for this paper. \citet{locatelli2016optimal}  
prove that up to multiplicative constants, and additives $\log(TK)$ terms, in the exponential, the optimal probability of regret in this problem is $\exp(- \frac{T}{\sum_{i:\Delta_i> 0} \Delta_i^{-2}})$,
where $\Delta_i = |\tau - \mu_i|$.
 We present their results for completeness and comparison to the bounds under additional shape constraints in Table \ref{tab:PD1} - see also Subsection~\ref{ss:unstr}. The \tbp in the problem dependent regime is also studied by \citet{mukherjee2017thresholding} and \citet{JieZhong17bandits}, however they consider a problem complexity based also upon variance making their results not so relevant to our setting. The \textit{problem independent} regime for the \tbp is studied by \citet{cheshire2020}, we also present their results in Table~\ref{tab:PD1} for comparison across the different regimes. 
\vspace{-0.4cm}
\paragraph{Monotone constraint, \stbp.} We then consider the problem where on top of assuming that the distributions are sub-Gaussian, we assume that the sequence of means $(\mu_k)_{k\in[K]}$ is monotone - this is problem \stbp. This specific instance of the \tbp is introduced within the context of drug dosing by \citet{garivier2017thresholding}. In this paper, the authors provide an algorithm for the fixed confidence setting that is optimal asymptotically, in the fixed confidence regime. However the definition of the algorithms, as well as the provided optimal error bound, are defined in an implicit way and not so easy to relate in a simple way to the gaps $\Delta_i$ moreover it is not clear how to translate a result from the fixed confidence setting to the fixed budget one. On the other hand, the shape constraint on the means of the arms implies that the \stbp is related to \textit{noisy binary search}, i.e.~inserting an element into its correct place within an ordered list when only noisy labels of the elements are observed, see~\citet{feige1994computing}. They describe an algorithm structurally similar to ours, using a binary tree with infinite extension however they consider a simpler setting where the probability of correct labeling is fixed as some $\delta > \frac{1}{2}$ and go on to show that there exists an algorithm that will correctly insert an element with probability at least $1-\delta$ in $\mathcal{O}\left(\log\left(\frac{K}{\delta}\right)\right)$ steps. For further literature on the related yet different problem of noisy binary search, see \citet{feige1994computing}, \citet{ben2008bayesian}, \citet{emamjomeh2016deterministic}, \citet{Nowak09binary}. Again, these papers consider settings with more structural assumptions than our own and are focused on the problem independent, fixed confidence regime. The \textit{problem independent} regime for the \stbp is studied by \citet{cheshire2020}, we also present their results in Table~\ref{tab:PD1} for comparison across the different regimes. 

In this work, we prove that, up to universal multiplicative constants and additive $\log(K)$ terms in the exponential, the optimal error probability is $\exp( - T\min_k \Delta_k^2),$ which highlights the somewhat surprising fact that this structured monotone \tbp problem is akin to a one armed \tbp - see Subsection~\ref{ss:m}. We provide the Problem Dependent Monotone \tbp (\explore) algorithm that matches this bound, see Section~\ref{sec:alg}.\vspace{-0.4cm}

\paragraph{Concave constraint, \ctbp.}
We next consider the problem where on top of assuming that the distributions are sub-Gaussian, we assume that the sequence of means $(\mu_k)_{k\in[K]}$ is concave - this is problem \ctbp. Again, in the problem independent regime the \ctbp has been studied by \citet{cheshire2020}. In the problem dependent regime however, to the best of our knowledge, the \ctbp has not been studied in the literature. However the related problems of estimating a concave function and optimising a concave function are well studied in the literature. Both problems are considered primarily in the continuous regime which makes comparison to the $K$-armed bandit setting difficult. The problem of estimating a concave function has been thoroughly studied in the noiseless setting, and also in the noisy setting, see e.g.~\citet{simchowitz2018adaptive}, where a continuous set of arms is considered, under H\"older smoothness assumptions. The problem of optimising a convex function in noise without access to its derivative - namely zeroth order noisy optimisation - has also been extensively studied. See e.g.~\citet{yudin}[Chapter 9], and~\citet{wang2017stochastic, agarwal2011stochastic, liang2014zeroth}
to name a few, all of them in a continuous setting with dimension $d$. The
focus of this literature is however very different to ours and \citet{cheshire2020}, as the main difficulty under their assumption is to obtain a good dependence in the dimension $d$, and with this in mind logarithmic factors are not very relevant.

In this work, we prove that, up to universal multiplicative constants and additive $\log(K)$ terms in the exponential, the optimal error probability is $\exp( - T\min_k \Delta_k^2),$
which highlights the somewhat surprising fact that this structured concave \tbp problem is also akin to a one armed \tbp  - see Subsection~\ref{ss:c}. We provide the Problem Dependent Concave \tbp (\CTB) algorithm that matches this bound, see Section~\ref{sec:alg}.\vspace{-0.2cm}

\paragraph{Organisation of the paper} This paper is structured as follows. In Section~\ref{sec:setting} we formally introduce the \tbp setting along with the monotone and concave shape constraints. We also describe the performance criterion - probability of error, we will be primarily using for the duration of the paper. Following this, upper and lower bounds on probability of error for all shape constraints are presented in Section~\ref{sec:minmaxthm}. Descriptions of algorithms achieving said upper bounds can be found in Section~\ref{sec:alg}. The results are discussed and compared to related work in Section~\ref{sec:discussion}. In Appendix~\ref{app:experiments} we conduct some preliminary experiments to explore how our theoretical results translate in practice. All proofs are found in the Appendix.
\vspace{-.4cm}
% !TEX root = ../main.tex
\section{Setting}
\label{sec:setting}

\paragraph{Problem formulation} The learner is presented with a $K$-armed bandit problem $\unu =\{\nu_1,\ldots,\nu_K\}$, with $K\geq 3$, where $\nu_k$ is the unknown distribution of arm $k$.

Let $\sigma^2 \geq 0$. We remind the learner that distribution $\nu$ of mean $\mu$ is said to be $\sigma^2$-sub-Gaussian if for all $t \in \R$ we have,
\vspace{-.2cm}
$$\EE_{X \sim \nu}\big[e^{t\left(X - \mu\right)}\big] \leq \mathrm{exp}\left(\frac{\sigma^2 t^2}{2}\right)\,.$$
In particular the Gaussian distributions with variance smaller than $\sigma^2$ and the distributions with absolute values bounded by $\sigma$ are $\sigma^2$-sub-Gaussian.

Let $\cB:= \cB( K, \sigma^2)$ be the set of all bandit problems as presented above, i.e.~where the distributions $\nu_k$ of the arms are all $\sigma^2$ sub-Gaussian.

In what follows, we assume that all $\unu \in \cB$, and we write $\mu_k$ for the mean of arm $k$. Let $\tau\in \mathbb R$ be a fixed threshold known to the learner. We aim to devise an algorithm which classifies arms as above or below threshold $\tau$ based on their means. That is, the learner aims at finding the vector $Q\in\{-1,1\}^K$ that encodes the true classification, i.e. $Q_k = 2\ind_{\{\mu_k \geq \tau \}}-1$  with the convention $Q_k = 1$ if arm $k$ is above the threshold and $Q_k = -1$ otherwise. The \emph{fixed budget} bandit sequential learning setting goes as follows: the learner has a budget $T>0$ and at each round $t \leq T$, the learner pulls an arm $k_t\in [K]$ and observes a sample $Y_{t}\sim \nu_{k_t}$, conditionally independent from the past. After interacting with the bandit problem and expending their budget, the learner outputs a vector $\hQ\in\{-1,1\}^K$ and the aim is that it matches the unknown vector $Q$ as well as possible.

\vspace{-0.4cm}

\paragraph{Unstructured case \tbp} 
In the \textit{problem dependent} regime, for $\bD \in \R^K_+$, we consider the following class of problems
\[\cdB = \{\nu \in \cB : \forall k \in [K], \; |\mu_k - \tau| =  \bD_k\}\;.\]

\paragraph{Monotone case \stbp} 

We denote by $\Bs$ the set of bandit problems,
$$\Bs :=  \{\nu \in \cB:\ \mu_1 \leq\mu_2 \leq  \ldots \leq\mu_K\}\;,$$
where the learner is given the  additional information that the sequence of means $\left( \mu_k \right)_{k \in [K]}$ is a monotonically increasing sequence. We denote by $\DeltaBs = \{\bD \in \R^K_+: \exists \nu\in\Bs, \forall k\in[K],\, |\mu_k-\tau|=\bD_k\} $ the set of possible vectors of gaps in $\Bs$ - i.e.~the set of sequences $\bD$ that would correspond to at least one problem in $\Bs$. In the \textit{problem dependent} regime, for $\bD \in \DeltaBs$, we consider the following class of problems
\[\Bsp = \{\nu \in \Bs : \forall k \in [K], \; |\mu_k - \tau|= \bD_k\}\;.\vspace{-0.2cm}\]

\paragraph{Concave case \ctbp}
We will denote by $\Bc$ the set of bandit problems, 
\[\Bc := \left\{ \nu \in \cB: \forall 1<k< K-1, \frac{1}{2}\mu_{k-1} + \frac{1}{2}\mu_{k+1} \leq \mu_k\right\}\,, \]
where the learner is given the  additional information that the sequence of means $\left( \mu_k \right)_{k \in [K]}$ is concave. We denote by $\DeltaBc = \{\bD \in \R^K_+: \exists \nu\in\Bc, \forall k\in[K],\, |\mu_k-\tau|=\bD_k, \, \exists l: \mu_l \geq \tau\} $ the set of possible vectors of gaps in $\Bc$ where at least one arm is above threshold - i.e.~the set of sequences $\bD$ that would correspond to at least one problem in $\Bc$ where at least one arm is above threshold. %\al{I modified the set here to have at least one arm up threshold. Btw I rewrote the UB theorems to get rid of $\bD$ - they are now expressed in function of $\Delta$ for any unimodal problem.} 
In the \textit{problem independent} regime, for $\bD \in \DeltaBc$, we consider the following class of problems
\[\Bsc := \left\{\nu \in \Bc : \forall k <K, |\mu_k - \tau| \in \left[\frac{\bD_k }{2},3\frac{\bD_k }{2}\right]\right\}\,.\]

\begin{remark}
The classes of problems $\cdB, \Bsp, \Bsc$ contain bandit problems in resp.~$\cB,\Bs, \Bc$ that are `local' around $\bD$ in the sense that while the sign of $\mu_k - \tau$ is arbitrary - although severely restricted by the shape constraint when it comes to $ \Bsp, \Bsc$ - the gap of arm $k$ is fixed to being - approximately, for the concave case set $\Bsc$ - $\bD_k$. This implies that in each case and on top of the respective shape constraint, we restrict ourselves to a small class of problems whose complexity is entirely characterised by $\bD$, in a \textit{problem dependent sense}.\vspace{-0.2cm}
\end{remark}

\paragraph{Strategy} A strategy is a sequence of functions that maps the information gathered in the past to an arm and finally to a classification. Precisely, if we denote by $I_t$ the information available to the player at time $t$, that is $I_t = \left\{ Y_1,Y_2,\ldots, Y_{t}\right\}$,
with the convention $I_0=\emptyset$. Then a strategy $\pi = \big((\pi_t)_{t\in[T]}, \hQ^\pi \big) $ is given by a sampling rule $\pi_t(I_{t-t}) = k_t \in [K]$ and a classification rule 
$\hQ^\pi(I_T) = \hQ \in \{-1,1\}^K$.

\vspace{-0.3cm}
\paragraph{Minimax expected regret}
The \textit{problem independent},  \textit{fixed budget} objective of the learner following the strategy $\pi$ is then to minimize the expected simple regret of this classification for $\hat Q:= \hat Q^\pi$:
\[r_{T}^{\unu,\pi} = \EE_{\unu} \!\left[\max_{\{k\in[K]:\ \hQ_k^\pi \neq Q_k\}} \Delta_k \right],\] 
where $\Delta_k : = |\tau - \mu_k|$ is the gap of arm $k$, and where $\EE_{\unu}$ is defined as the expectation on problem $\unu$ and $\mathbb P_{\unu}$ the probability. However, the focus of this paper is on the \textit{problem dependent} regime where, as usual, we consider as a performance criterion rather the related \textit{probability of error}
\vspace{-.2cm}
\[e_{T}^{\unu,\pi} = \PP_{\unu} \!\left(\exists k \in [K] : \hQ_k^\pi \neq Q_k\right)\,.\] 
When it is clear from the context we will remove the dependence on the bandit problem $\unu$ and/or the strategy $\pi$. Note that if we denote by $\Dm = \min_{k\in[K]} \bD_k$ the minimum of the gaps then 
\[
r_{T}^{\unu,\pi} \geq \Dm e_{T}^{\unu,\pi}\,.
\]
Consider a set of bandit problems $\tilde{\mathcal B}\subset \cB$. The minimax optimal probability of error on $\tilde B$ is then
$$e_T^*(\tilde{\mathcal B}) := \inf_{\pi~{\mathrm{strategy}}} \sup_{\unu \in \tilde{\mathcal B}} e_T^{\unu,\pi}.$$
We will study this quantity over the local classes $\cdB, \Bsp, \Bsc$.

\begin{remark}
As argued above, the classes $\cdB, \Bsp, \Bsc$ contain only bandit problems that satisfy their respective shape constraint and whose complexity is entirely characterised by $\bD$, in a \textit{problem dependent sense}. Studying the minimax probability of error over these very restricted classes is therefore a very meaningful way of studying the problem dependent regime of structured \tbp problems - and we expect this probability of error to heavily depend on $\bD$. The focus of this paper is to characterise this dependence in a tight manner.\vspace{-0.2cm}
\end{remark}

% When presenting our results we will use the log mini max regret defined as follows, 

% $$L\tilde{R}_T^*(\tilde{\mathcal B}) := \inf_{\pi~{\mathrm{strategy}}} \sup_{\unu \in \tilde{\mathcal B}}\tilde R_{T}^{\unu,\pi}\,.$$

% !TEX root = ../main.tex
\section{Minimax rates}
\label{sec:minmaxthm}
In this section we present upper and lower bounds on probability of error for all three shape constraints. Given a vector $\bD \in \mathbb{R}^K_+$ we denote $\Dm = \min_{k\in[K]} \bD_k$.
\vspace{-.2cm}
\subsection{Problem dependent unstructured setting \tbp} \label{ss:unstr}

The unstructured thresholding bandit in the problem dependent regime has already been considered in the literature. We remind results from~\citet{locatelli2016optimal}, where they provide tight upper and lower bounds over $e_T^*(\cdB)$, for any $\bD \in \R^{K}_+$. In our context they prove that
\begin{small}
\begin{align*}
\exp\!\left(-\frac{3}{\sigma^{2}} \frac{T}{H} - 4\sigma^{-2} \log \left(12 (\log T +1) K\right)\right)&\leq e_{T}^*(\cdB) \\
\leq  \exp\!\Big(-\frac{1}{64 \sigma^2}\frac{T}{H} +2 \log &\left((\log T +1) K\right)\!\!\Big),    
\end{align*}
\end{small}
where $H = \sum_{i: \bD_i>0} 1/\bD_i^2$ - see Theorems~1 and~2 by~\citet{locatelli2016optimal}. This implies that up to multiplicative universal constants and whenever $T \geq H\sigma^2 \log(\log(T)+K)$, it holds that
\vspace{-.2cm}
$$-\log\left(e_{T}^*(\cdB) \right) \asymp \frac{1}{\sigma^{2}} \frac{T}{H},$$
and upper and lower bound match up to universal multiplicative constants in the exponential of the error probability. The quantity $H$ is therefore the problem dependent quantity that characterises the difficulty of the problem. Note that of course, the APT algorithm by~\citet{locatelli2016optimal} does not take any information on the class - $\bD$, but also $\sigma^2$ - as parameters, and is essentially parameter free.

In this paper, we won't therefore discuss further this unstructured setting - the reminder provided here is only to be taken as a benchmark for the rest of the paper. We will on the other hand focus on the structured problems - monotone and concave and study how the minimax error probability evolves, in particular depending on the problem-dependent quantities $\bD$.
\vspace{-0.2cm}
%From \cite{locatelli2016optimal} one has matching upper and lower bounds, up to a logarithmic term, for the \tbp in the problem dependent regime. We will not consider the unstructured setting in this paper. See Table \ref{tab:PD1} for a the exact rates of \cite{locatelli2016optimal} and an analysis of their results in the context of our own. 

\subsection{Problem dependent monotone setting} \label{ss:m}

Given a class of problems $\Bsp$ for some $\bD \in \DeltaBs$, the following theorem provides a lower bound on the probability of error for any strategy $\pi$. The proof of Theorem \ref{thm:depmon_down} can be found in Appendix \ref{app:proof_monotone}.
\begin{theorem}\label{thm:depmon_down}
Let $\bD \in\DeltaBs$. For any strategy $\pi$ there exists a monotone bandit problem $\unu \in \Bsp$ such that \vspace{-1mm}
$$e_{T}^{\unu,\pi} \geq   \frac{1}{4}\exp 
\!\left(-\frac{T\Dm^2}{\sigma^2}\right)\,.$$
\end{theorem}
 Now the following theorem gives an upper bound on the probability of error for the \explore algorithm. The proof of Theorem~\ref{thm:depmon_up} can be found in Appendix \ref{app:proof_monotone}.
\begin{theorem}\label{thm:depmon_up} 
Let $\nu \in\Bs$ associated with arm gaps $\Delta$, and assume that $T > 36\log(K)$. %For all $\nu \in \Bsp$ %and $ \Dm > \cmin\sqrt{\frac{\sigma^2 \log K}{T}}$, 
The algorithm \explore satisfies the following bound on error probability:
$$e_{T}^{\unu,\explore} \leq \exp\!\left(-\cmon \frac{T\Delta_{min}^2}{\sigma^2}+\cmon' \log(K) \right)$$
where $\cmon=1/48$ and $\cmon' =12$.  
\end{theorem}
The parameter free algorithm \explore is described in Sections~\ref{sec:alg} - see also Appendix \ref{app:proof_monotone}.

The assumption on $T$ is reasonable as in the monotone setting it is clear no algorithm can gain enough information in less than $\log(K)$ pulls, see~\citet{cheshire2020}. Note that combining both bounds yields that whenever $T > 36\log(K)/\Dm^2$:
\vspace{-.2cm}
$$-\log\left(e_T^*(\Bsp) \right) \asymp \frac{1}{\sigma^{2}} T \Dm^2,$$

and upper and lower bound match up to universal multiplicative constants in the exponential of the error probability. Perhaps surprisingly, the number of arms plays no role in this rate - as long as we assume that $T > 36\log(K)/\Dm^2$. Only the minimal arm gap appears, and this amounts to saying that when $T > 36\log(K)/\Dm^2$, this problem is not more difficult - in order, up to universal multiplicative constants in the exponential - than a one-armed \tbp with gap $\min_k \Delta_k$! And that in a sense, even if we knew in our monotone problem the position of all means but one - the arm with minimal gap - with respect to the threshold, the problem would not be significantly easier.

\vspace{-0.2cm}

%Note that if $\min_k \Delta_k \lesssim \sqrt{\frac{\log(K)}{T}}$, the upper bounds stops being informative as it is larger than 1. This is not surprising when one considers what the global minimax rate for the simple regret in this problem is precisely $\sqrt{\frac{\log(K)}{T}}$, see Table \ref{tab:PD1} and \cite{cheshire2020}. Essentially this means that for $\min_k \Delta_k \lesssim\sqrt{\frac{\log(K)}{T}}$ any algorithm will have expected simple regret greater than $\Dm$ and will therefore make a mistake when classifying the set of arms with high probability. \al{Might be better to keep this one for the discussion no? In particular when we compare problem dep and problem indep}

% \begin{theorem}\label{thm:depmon_minmax} 
% \al{Define $\Dm$ here}
% For all $\bD \in \mathbb{R}_+^K$ such that $ \Dm > \cmin\sqrt{\frac{\sigma^2 \log K}{T}}$, it holds that
% $$\tilde R_T^*(\Bsp) \asymp \exp\left(- \frac{T\Dm^2}{\sigma^2} \right).$$
% The algorithm \explore described in Sections~\ref{sec:alg} (see also Appendix~\ref{app:proof_mondep}) attains this rate. \al{Explore is not such a good name... rather MTB?}
% \end{theorem}

\subsection{Problem dependent concave setting}\label{ss:c}

Given a class of problems $\Bsc$ for some $\bD \in \DeltaBc$ the following theorem provides a lower bound on the probability of error for any strategy $\pi$. The proof of Theorem \ref{thm:depcon_down} can be found in Appendix \ref{app:proof_concave}.
\begin{theorem}\label{thm:depcon_down} 
Let $\bD \in \DeltaBc$. For any strategy $\pi$ there exists a problem $\nu \in \Bsc$ such that
$$e_{T}^{\unu,\pi} \geq    %\min\left(\frac{1}{2},
\frac{1}{4}\exp \left(-9\frac{T\Dm^2}{\sigma^2}\right)
%\right)
\;.$$
\end{theorem}
Now the following theorem gives an upper bound on the probability of error for the \CTB algorithm. The proof of Theorem~\ref{thm:depcon_up} can be found in Appendix~\ref{app:proof_concave}.

\begin{theorem}\label{thm:depcon_up} 
Let $\nu \in \Bc$ with associated gaps $\Delta$ and assume $T > 108\log(K)$. %For all $\nu \in \Bsc$ t
The algorithm \CTB has the following bound on error, 
$$e_{T}^{\unu,\CTB} \leq 3\exp\left(-\ccon  \frac{T \Delta_{\min}^2}{\sigma^2}+\ccon'\log(K)\right) $$
where $\ccon=1/576$ and $\ccon' =12$.  
\end{theorem}
\vspace{-0.1cm}

The parameter free algorithm \CTB is described in Sections~\ref{sec:alg} - see also Appendix~\ref{app:proof_concave}.

The assumption on $T$ is reasonable as in the monotone setting it is clear no algorithm can gain enough information in less than $\log(K)$ pulls, see~\citet{cheshire2020}. Note that combining both bounds yields that whenever $T > 108\frac{\log(K)}{\Dm^2}$:
\vspace{-.3cm}
$$-\log\left(e_T^*(\Bsp) \right) \asymp \frac{1}{\sigma^{2}} T \Dm^2,\vspace{-0.3cm}$$

and upper and lower bound match up to universal multiplicative constants in the exponential of the error probability. Similar comments can be made here as in the case of the monotone \tbp in Section~\ref{ss:m}: the convex \tbp is also as difficult as a one-armed \tbp with gap $\min_k \Delta_k$. 
\vspace{-0.4cm}
%Perhaps surprisingly, the number of arms plays no role in this rate - as long as we assume that $T > 36\log(K)$. Only the minimal arm gap appears, and this amounts to saying that when $T > 36\log(K)$, this problem is not more difficult - in order, up to universal multiplicative constants in the exponential - than a one-armed \tbp with gap $\min_k \Delta_k$! And that in a sense, even if we knew in our monotone problem the position of all means but one - the arm with minimal gap - with respect to the threshold, the problem would not be significantly easier...

% \begin{theorem}\label{thm:depcon_minmax} 

% For all $\Delta \in \mathbb{R}_+$ such that $\Dm = \min_{i\in[K]}\left(\Delta_i\right) > \cmin\sqrt{\frac{\sigma^2 \log K}{T}}$, it holds that
% \[\tilde R_T^*(\Bsc) \asymp \exp\left(- \frac{T\Dm^2}{\sigma^2} \right).\]
% The algorithm \CTB described in Sections~\ref{sec:alg} (see also Appendix~\ref{app:proof_condep}) attains this rate.
% \end{theorem}
% !TEX root = ../main.tex
\section{Optimal algorithms in the problem dependent regime}
\label{sec:alg}

\subsection{Monotone case \stbp} \label{sec:mon} 

%In this section we fix a problem $\unu \in \Bsp$ for some $\bD \in \DeltaBs$ with $\min_k \bD_k \geq \cmin\sqrt{\frac{\sigma^2 \log K}{T}}$. 
We assume in this section, without loss of generality, instead of considering $K$ arms, we consider for technical reasons $K+2$ arms adding two deterministic arms $0$ and $K+1$ with respective means $\mu_0 = -\infty$ and $\mu_{K+1} = +\infty.$ While we assume that the distributions of the original $K$ arms are $\sigma^2$-sub-Gaussian the addition of two such arms will not invalidate our proofs, see Appendix  \ref{app:proof_monotone}. 
 We do this to ensure that, after re-indexing of the arms and adapting the number of arms, $\tau \in [\mu_1, \mu_{K}]$.

To match a minimax rate as described in Section \ref{sec:minmaxthm} we will utilise a modified version of the MTB algorithm described by \citet{cheshire2020}. The algorithm \explore performs a random walk on the set of arms $[K]$ as a binary tree. We consider the binary tree as \citet{cheshire2020} with an specific extension akin to that by \citet{feige1994computing}. 

\paragraph{Binary Tree}\label{sec:bin}
We associate to each problem $\unu \in \Bs$ a binary tree. Precisely we consider a binary tree with nodes of the form $v=\{L,M,R\}$ where $\{L,M,R\}$ are indexes of arms and we note respectively $v(l)=L, v(r)=R, v(m)=M$. The tree is built recursively as follows: the root is $\root =\{1,\floor*{(1+K)/2}, K\}$, and for a node $v=\{L,M,R\}$ with $L,M,R\in\{1,\ldots,K\}$ the left child of $v$ is $L(v)= \{L,M_l,M\}$ and the right child is $R(v)=\{M,M_r,R\}$ with $M_l = \floor*{(L+M)/2}$ and $M_r = \floor*{(M+R)/2}$ as the middle index between. The leaves of the tree will be the nodes $\{v = \{L,M,R\} : R = L +1\}$. If a node $v$ is a leaf we set $R(v) = L(v) = \emptyset$.  We consider the tree up to maximum depth $H= \floor*{\log_2(K)}+1$. We note $P\big(l(v)\big)=P\big(r(v)\big)$ the parent of the two children and let $|v|$ denote the depth of node $v$ in the tree, with $|\root| = 0$. We adopt the convention $P(\root) = \root$.

\vspace{-0.3cm}
\paragraph{Extended Binary Tree}
 We extend the above Binary tree in the following manner. For a leaf $v$ we replace the condition $R(v) = L(v) = \emptyset$ with the following: for any leaf $v = \{L,M,R\}$ we set $R(v) = \tilde v$ where $\tilde v = \{L,M,R\}$ and set $L(v) = \emptyset$. Note that $\tilde v$ is also a leaf therefore iterative application this relation will lead to an infinite extension. The result being that each leaf in our original binary tree is now the root of an infinite chain of identical nodes, see Figure~\ref{fig:binary_tree}. For practical purposes we need only consider such an extension up to depth $T$ and can simply cut the tree at this depth. 

\begin{remark}
We set $L(v) = \emptyset$ for some leaf $v$ during the extension of the binary tree as by construction all leaves of the original binary tree are of the form $\{v = \{L,M,R\} : R = L +1 \;\mathrm{and}\;M=L\}$.
\end{remark}
\vspace{-0.2cm}

In order to predict the right classification we want to find the arm whose mean is the one just above the threshold $\tau$. Finding this arm is equivalent to inserting the threshold into the (sorted) list of means, which can be done with a binary search in the aforementioned binary tree. But in our setting we only have access to estimates of the means which can be very unreliable if the mean is close to the threshold. Because of this there is a high chance we will make a mistake on some step of the binary search. For this reason we must allow \explore to backtrack and this is why \explore performs a binary search \emph{with corrections}. 

\paragraph{\explore algorithm} 
First, define the following integers 
\vspace{-.1cm}
\begin{equation}\label{eq:t1t2}
T_1 := \lceil 6 \log(K) \rceil \qquad T_2 := \floor*{\frac{T}{3T_1}}\;.
\end{equation}\vspace{-.5cm}
%\[T_2 = \frac{T}{3\log K} \qquad T_1 = \frac{T}{T_2}\]

The algorithm \explore is then essentially a random walk on said binary tree moving one step per iteration for a total of $T_1$ steps. Let $v_1 = \root$ and for $t < T_1$ let $v_t$ denote the current node, the algorithm samples arms $\{v_t(j) : j \in \{l,m,r\}\}$ each $T_2$ times. Let the sample mean of arm $v_t(j)$ be denoted $\hat{\mu}_{j,t}$. \explore will use these estimates to decide which node to explore next. If an error is detected - i.e. the interval between left and rightmost 
sample mean does not contain the threshold, then the algorithm backtracks to the parent of the current node, otherwise \explore acts as the deterministic binary search for inserting the threshold $\tau$ in the sorted list of means. More specifically, if there is an anomaly, $\tau \not \in \left[\hat{\mu}_{l,t},\hat{\mu}_{r,t}\right]$, then the next node is the parent $v_{t+1} = P(v_t)$, otherwise if $\tau \in \left[\hat{\mu}_{l,t},\hat{\mu}_{m,t}\right]$ the the next node is the left child $v_{t+1} = L(v_t)$ and if $\tau \in \left[\hat{\mu}_{m,t},\hat{\mu}_{r,t}\right]$ the next node is the right child  $v_{t+1} = R(v_t)$. If at time $t$, $\tau \in \left[\hat{\mu}_{l,t},\hat{\mu}_{r,t}\right]$ and the node $v_t$ is a leaf, that is $v(r) = v(l)+1$, then due to the extension of our binary tree $R(v_t) = L(v_t) = \tilde v_t$ where $\tilde v$ is a duplicate of $v_t$. Hence $v_{t+1} = \tilde v_t$. Via this mechanism the \explore algorithm essentially gives additional preference the the node $v_t$. See \explore for details. We now formally state the parameter free \explore algorithm (Problem Dependent Monotone Thresholding Bandit Algorithm). We rely on the assumption $T > 36\log(K)$, see Theorem \ref{thm:depmon_up} to ensure $T_2 \geq 1$.\vspace{-0.2cm}

\begin{algorithm}[H]
\caption{\explore}
\label{alg:explore} 
\begin{algorithmic}
\STATE {\bfseries Initialization:} $v_1=\root$
\FOR{$t=1:T_1$}
\STATE sample $T_2$ times each arm in $v_t$\\
\IF{$ \tau \not\in[\hat{\mu}_{l,t},\hat{\mu}_{r,t}]$}
\STATE $v_{t+1} = P(v_t)$ 
\ELSIF{$\hat{\mu}_{m,t} \leq \tau \leq \hat{\mu}_{r,t}$}
\STATE$v_{t+1} = R(v_t)$
\ELSIF{$\hat{\mu}_{l,t} \leq \tau \leq \hat{\mu}_{m,t}$}
\STATE $v_{t+1} = L(v_t)$
\ENDIF
\ENDFOR
\STATE Set $\hk = v_{T_1 + 1}(r)$\\
\RETURN $(\hk, \hQ) :\quad  \hQ_k=  2\ind_{\{k \geq \hk \}}-1$
\end{algorithmic}
\end{algorithm}

\begin{figure}[H]
    \centering
    \includegraphics[width =0.8\linewidth]{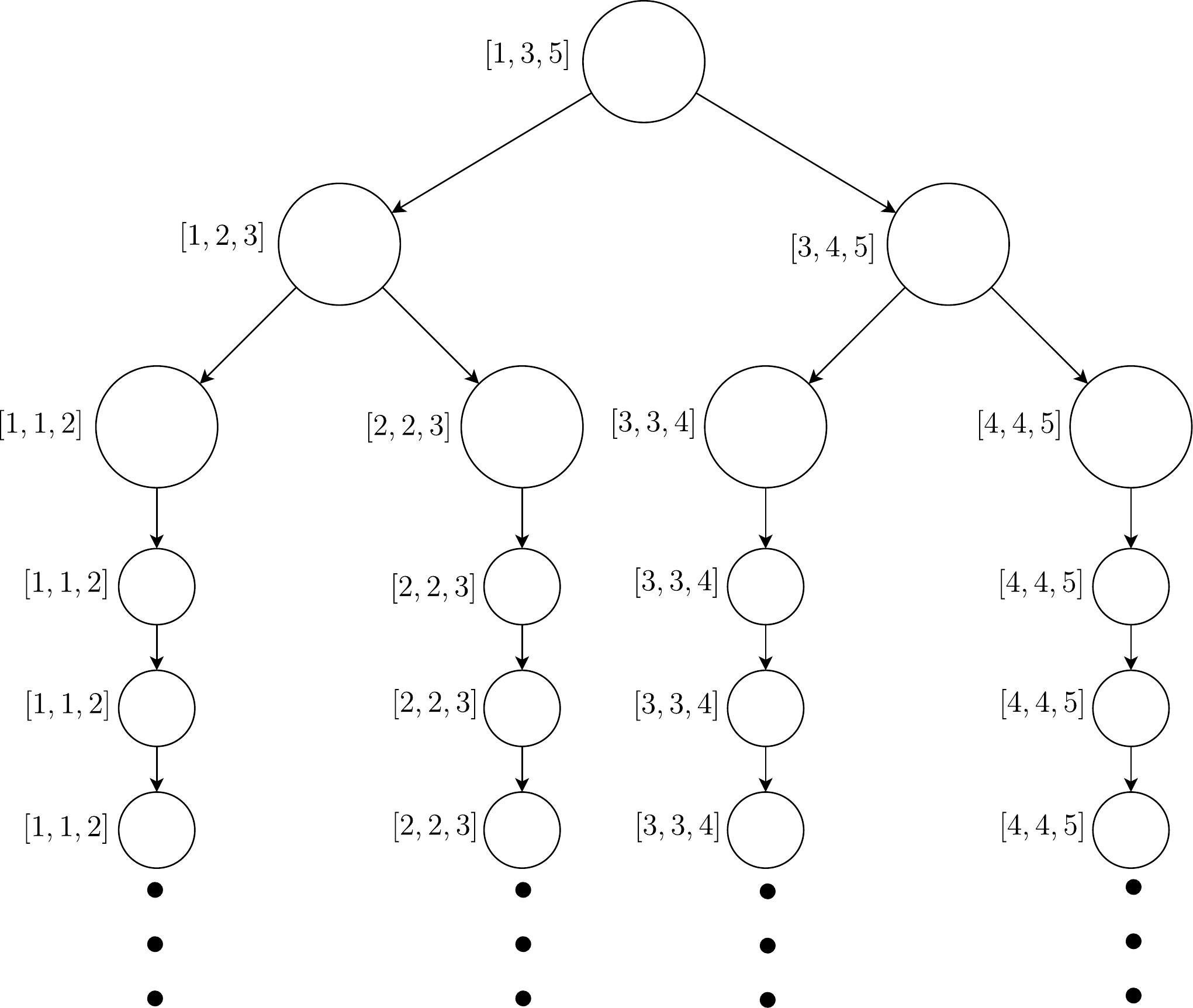}

    \caption{Extended binary tree for $K = 5$ }
    \label{fig:binary_tree}
\end{figure}\hfill
\vspace{-1cm}
\begin{remark}[Adaptation of \explore to a non-increasing sequence, \dexplore]\label{rem:Dexplore}
\explore is applied for a monotone non-decreasing sequence $(\mu_k)_{k\in[K]}$, and it is easy to adapt it to a monotone non-increasing sequence $(\mu_k)_{k\in[K]}$. In this case, we transform the label of arm $k$ into $K-k$, and apply \explore to the newly labeled problem - where the mean sequence in now non-decreasing. We refer to this modification as \dexplore.
\end{remark}

\begin{remark} [Relaxing the monotone assumption]\label{rem:relaxing_monotone} By inspecting the proof of Theorem~\ref{thm:depmon_up} in Appendix~\ref{app:proof_monotone} 
we can obtain the same guarantee for a larger class of problem than one with increasing means. Indeed we only need that there exists an arm for which all the arms before it have a mean below the threshold and all arm after have a mean above the threshold. Precisely the bound of Theorem~\ref{thm:depmon_up} holds also for problems that belongs to 
\begin{align*}
\Brs :=  &\{\nu \in \cB:\ \exists k \in [1,K],\, \forall j\leq k\ \mu_j\leq \tau, \\
&\quad\, \, \forall j\geq k+1\ \mu_j\geq \tau\}\;.
\end{align*}
\vspace{-0.8cm}

Note the same remark also applies for problems with monotone non-increasing sequence.
\end{remark}

\subsection{Concave case \ctbp}\label{sec:con}

We assume in this section, without loss of generality, instead of considering $K$ arms, we consider for technical reasons $K+2$ arms adding two deterministic arms $0$ and $K+1$ with respective means $\mu_0 = \mu_{K+1} = -\infty.$ While we assume that the distributions of the original $K$ arms are $\sigma^2$-sub-Gaussian the addition of two such arms will not invalidate our proofs, see Appendix \ref{app:proof_concave}.
We do this to ensure that after re-indexing $\tau> \mu_1,\mu_K$.

As in the monotone case we construct a binary tree to span the arms of the bandit problem.  %associate with each problem $\nu \in \Bsc$ a binary tree. 
The construction of this tree is identical to that described in Section \ref{sec:mon} but without the infinite extension. We will use a variant off the \explore Algorithm, \gradexplore to move around the tree. The difference is that \gradexplore bases its movement off the estimated gradients of the arms as opposed to their sample means. The objective of \gradexplore is to find an arm with corresponding mean above threshold. Once such an arm has been identified we split our problem into two ``relaxed monotone" bandit problems - see Remark \ref{rem:relaxing_monotone}, one increasing and one decreasing. We then run \explore and \dexplore respectively. We split our budget evenly across the three algorithms: \gradexplore, \explore and \dexplore.

\paragraph{\gradexplore algorithm}
As with \explore the algorithm \gradexplore is essentially a random walk on the said binary tree moving one step per iteration for a total of $T_1$ steps. Let $v_1 = \root$ and for $t < T_1$ let $v_t$ denote the current node, the algorithm samples arms $\{v_t(l),v_t(l)+1,v_t(m),v_t(m)+1,v_t(r),v_t(r)+1\}\}$ each $T_2$ times. As in Section~\ref{sec:mon}, we adopt the convention that the arm $K+1$ is a Dirac distribution at $-\infty$. Let the sample mean of arm $v_t(j)$ be denoted $\hat{\mu}_{j,t}$ and the sample mean of arm $v_t(j) + 1$ be denoted $\hmu_{j+1,t}$. Let the estimated local gradient at arm $j$, that is $\hmu_{j,t} - \hmu_{j+1,t}$ denote $\hnabla_{j,t}$. \gradexplore will use these estimates to decide which node to explore next. If an error is detected - i.e. the left most or right most gradient is negative or positive respectively, then the algorithm backtracks to the parent of the current node, otherwise \gradexplore acts as the deterministic binary search for the maximum mean, $\max_{i\in[K]}\mu_i$. More specifically, if there is an anomaly, $\left(\hnabla_{l,t},\hnabla_{r,t}\right) \notin \left(\mathbb{R}_+,\mathbb{R}_-\right)$, then the next node is the parent $v_{t+1} = P(v_t)$, otherwise if $\hnabla_{m,t} < 0$ the next node is the left child $v_{t+1} = L(v_t)$ and if $\hnabla_{m,t} \geq 0$ the next node is the right child  $v_{t+1} = R(v_t)$. 
See Algorithm~\ref{alg:gradexplore} for details.
\vspace{-0.5cm}
\noindent
\begin{minipage}[t!]{0.49\textwidth}
\vfill
\begin{algorithm}[H]
\caption{\gradexplore}
\label{alg:gradexplore} 
\begin{algorithmic}
\STATE {\bfseries Initialization:} $v_1=\root$\\
\FOR{$t=1:T_1$}
\STATE $S_{t+1} = S_t$
\STATE for each $k \in v_t$ sample $\frac{T_2}{12}$ times the arms $k,k+1$
\IF{$\exists k\in \{l,m,r\} : \hmu_k > \tau$}
\STATE Append arm $k$ to the list $S_{t+1}$\\
\STATE $v_{t+1} = v_t$
\ELSIF{$\left(\hnabla_{l,t},\hnabla_{r,t}\right) \notin \left(\mathbb{R}_+,\mathbb{R}_-\right)$}
\STATE $v_{t+1} = P(v_t)$
\ELSIF{$\hnabla_{m,t} \geq 0$}
\STATE $v_{t+1} = R(v_t)$
\ELSIF{$\hnabla_{m,t} < 0$}
\STATE$v_{t+1} = L(v_t)$
\ENDIF
\ENDFOR

\end{algorithmic}
%\Return  $\median(S_{T_1})$
\end{algorithm}
\end{minipage}
\hfill
\begin{minipage}[t!]{0.49\textwidth}
\vfill
\begin{algorithm}[H]
\caption{\CTB}
\label{alg:CTB} 
\begin{algorithmic}
\STATE \textbf{run} \gradexplore
\STATE \textbf{output} list $S_{T_1}$ 
\IF{$\left|S_{T_1}\right| \leq \frac{T_1}{4}$}
\RETURN $\hQ = \{-1\}^K$ 

\ELSE{
\STATE $\hk = \median(S_{T_1})$
\STATE $l$ = output of \dexplore on set of arms $[1,\hk]$ budget: $\frac{T}{3}$
\STATE $r$ = output of \explore on set of arms $[\hk,K]$ budget: $\frac{T}{3}$

\RETURN $\hQ: \quad \hQ_k = 1 - 2\ind_{k<l} - 2\ind_{k>r}$
}
\ENDIF
\end{algorithmic}
\end{algorithm}
\end{minipage}

For the arms whose means are below threshold, due to the concave property gradients are essentially greater than $\Dm$ and can easily be estimated. Above threshold however gradients are less than $\Dm$ and are relatively hard to estimate. Therefore, although on the face \gradexplore is in part a binary search for the arm with maximum mean, in reality this is not feasible. The true utility of \gradexplore to the learner is to act as a binary search for the "set" of arms above threshold. If we refer to nodes containing an arm $k:\mu_k > \tau$ as "good nodes" the idea behind $\gradexplore$ is to spend a sufficient amount of time in exploring this set of nodes and adding "good arms" - i.e ones with a corresponding mean above threshold, to the list $S$. We can then output such an arm with high probability when outputting the median of $S_{T_1}$.

Once we have identified our arm above threshold we split our problem into two bandit problems where the classification can be done by binary search, see Remark~\ref{rem:relaxing_monotone} and~\ref{rem:Dexplore}. We can thus then apply \explore and \dexplore. Precisely, the complete procedure, namely \CTB (Problem Dependent- Concave Threshold Bandits), is detailed in Algorithm~\ref{alg:CTB}.
\vspace{-.5cm}

% \begin{algorithm}[H]
% \caption{\CTB}
% \label{alg:CTB} 
% \SetAlgoLined
% \textbf{run} \gradexplore\newline
% \textbf{output} arm $\hk$ \newline
% Set $l$ = output of \dexplore on set of arms $[1,\hk]$\newline
% Set $r$ = output of \explore on set of arms $[\hk,K]$\newline
% \Return $\hQ: \quad \hQ_k = 1 - 2\ind_{k<l} - 2\ind_{k>r}$ 
% \end{algorithm}

%\input{main/unimodal}
% !TEX root = ../main.tex
\section{Discussion}
\label{sec:discussion}
\vspace{-0.1cm}
%\al{Here add a blabla on algorithms and results - say that we are optimal and in which sense, comment on the algorithms and on their absence of parameters, comment on how they relate to each other, comment on the class of problems, etc}

%\subsection{parameter-less algorithms}

\subsection{Algorithms \texorpdfstring{\explore}{} and \texorpdfstring{\CTB}{}}
Both the \explore and \CTB are based upon a binary search with corrections, this allows them to exploit the structure of the shape constraints reducing the problems to sets of arms with cardinally of order $\log(K)$, something in sharp contrast to existing algorithms for the vanilla setting. The difference between \explore and \CTB is that while \explore works exclusively on a binary tree based upon the classification of an arms mean above or below threshold, the sub algorithm \gradexplore of \CTB bases a binary tree on positive or negative gradient. Therefore \explore acts as a search for the point the arms cross threshold while \gradexplore acts as a search for the arm $k^* = \argmax_k(\bD_k)$. Another more subtle difference is that on a "good decision" at time $t$ - i.e when the sample means are well concentrated up to $\Dm$, \explore will make a step in the right direction. The same cannot be said for \gradexplore as we can only guarantee that the increments between arms are greater than $\Dm$ for arms below threshold, this is a direct result of the concave property. Therefore the true utility of \gradexplore is not to find $k^*$ but to find any arm $k:\mu_k >\tau$.  

It is worth noting that both algorithms described in this paper are parameter free, being adaptive not only to the hardness of the problem characterised by the gaps $\bD$, but also to the underlying sub-Gaussian assumption parameter $\sigma^2$. 
\vspace{-.2cm}
\subsection{Problem classes and optimality} 

In the monotone and concave settings we consider a very narrow class of problems and argue our classes are relevant for characterising the problem dependent regime - i.e.~are narrow enough. 

\begin{itemize}
    \item In the monotone setting this is obvious as the class of problems is defined by a specific vector $\bD \in \mathbb{R}^K_+$, so that all problems in this class have a similar complexity, bear in mind that our algorithms do not need to know $\Dm$ or any aspect of $\bD$. In fact, when constructing our lower bound, we just need a class with two problems where, given a first problem, we simply switch the arm with minimal gap $\Dm$ from below to above threshold in order to obtain the second problem - see the proof of Theorem~\ref{thm:depmon_down}.
    \item In the concave setting this approach is unfeasible as under the concave constraints the class of problems defined by a specific vector of gaps $\bD \in \mathbb{R}^K_+$ has very often cardinality 1 which is nonsensical for a lower bound. Instead, given a specific vector $\bD \in \mathbb{R}^K_+$ we consider a class of problems  with gaps within a proportional tolerance of $\bD$. This class is designed to be as narrow as possible while still containing multiple problems which disagree on the placement of certain arms above or below threshold. In fact, when constructing our lower bound, we just need a class with two problems where, starting from a first problem, we simply flip the arm with minimal gap and translate other means vertically in such a way to preserve concavity - see the proof of Theorem~\ref{thm:depmon_down}.
\end{itemize}
In both cases, we prove that for $T$ large enough, the problem dependent optimal probability of error is of order
\vspace{-.1cm}
$$\exp(- T \Dm^2/\sigma^2),\vspace{-0.3cm}$$

up to universal multiplicative constants inside and outside the exponential. This implies that from a problem dependent perspective, both problems are as difficult as a one armed bandit problem where we just want to decide whether the arm with minimal gap $\Dm$ is up or down the threshold, which is quite surprising - as the number of arms plays therefore no role asymptotically. While the lower bounds are relatively simple, the upper bounds are more interesting and challenging.
\vspace{-.3cm}

% In the unimodal setting our class of problems is more contrived and it is unclear whether there is room for improvement. It would again be favourable to consider a class based on a specific vector $\bD \in \mathbb{R}^K_+$ but modifying the lower bound for such an extension appears to be none trivial as it currently relies on "shifting" the segment of arms above threshold around. 

\subsection{Comparison of rates between settings}

Table \ref{tab:PD1} presents a comparison of results across the problem independent and dependent regimes. Although the results are not immediately comparable between the regimes, of particular interest is the difference in rates across the monotone and concave settings in the problem independent regime compared to the lack of difference between said rates in the problem dependent regime. 
% \begin{table}[!h]
% \begin{center}
% \renewcommand{\arraystretch}{1.7}
%  \begin{tabular}{c||c|c|c}
%   \;& Unstructured   & Monotone   & Concave    \\\hline \hline
% problem independent& $\sqrt{\frac{K\log K}{T}}$ & $\sqrt{\frac{\log K \lor 1}{T}}$ & $\sqrt{\frac{\log\log K \lor 1}{T}}$ \\ \hline
%  problem dependent & $\exp\left(-\frac{T}{H}\right)$ & $\exp\left(- T\Dm^2\right)$  & $\exp\left(- T\Dm^2\right)$\\
%  \end{tabular}
% \label{tab:PD1}
%  \caption{Order of the optimal problem dependent probability of error, and of the problem independent expected simple regret for the three structured \tbp, in the case of all four structural assumptions on the means of the arms considered in this paper. All results are given up to universal multiplicative constants both in and outside the exponential. The first line concerns the problem independent setting and the simple regret, see \cite{cheshire2020}. The second line concerns the problem dependent setting and the probability of error, the main focus of this paper. The results for the monotone and concave are novel and can be found in this paper, see Section \ref{sec:minmaxthm}. The results for the unstructured setting can be found in \cite{locatelli2016optimal}, where they take $H = \sum_{i=1}^{K} \bD_i^{-2}$}
% \end{center}
%  \vspace{-.9cm}
%  \end{table}
 
\begin{table}[!h]
\begin{center}
\renewcommand{\arraystretch}{1.7}
 \begin{tabular}{c||c|c|c}
  problem:& independent  & dependent \\    \hline \hline 
Unconstrained & $\sqrt{\frac{K\log K}{T}}$ &$\exp\left(-\frac{T}{H}\right)$\\ \hline

Monotone & $\sqrt{\frac{\log K \lor 1}{T}}$ &$\exp\left(- T\Dm^2\right)$  \\ \hline
Concave & $\sqrt{\frac{\log\log K \lor 1}{T}}$& $\exp\left(- T\Dm^2\right)$\\
% (for $T$ large enough)&&&&
 % problem dependent LB:& $\exp\left(-\frac{T}{H} - \log(\log(T)K)\right)$ & $\exp\left(-c_{con} \frac{T\Dm^2}{\sigma^2}\right)$ &  $\;$  & $\exp\left(-c_{con} \frac{T\Dm^2}{\sigma^2}\right)$\\
 \end{tabular}
\label{tab:PD1}
 \caption{Order of the optimal problem dependent probability of error, and of the problem independent expected simple regret for the three structured \tbp, in the case of all four structural assumptions on the means of the arms considered in this paper. All results are given up to universal multiplicative constants both in and outside the exponential. The first line concerns the problem independent setting and the simple regret, see \citet{cheshire2020}. The second line concerns the problem dependent setting and the probability of error, the main focus of this paper. The results for the monotone and concave are novel and can be found in this paper, see Section \ref{sec:minmaxthm}. The results for the unstructured setting are by \citet{locatelli2016optimal}, where they take $H = \sum_{i=1}^{K} \bD_i^{-2}$}
\end{center}
\vspace{-0.3cm}
 \end{table}

In both the monotone and concave setting an initial lower bound is one which does not depend upon $K$ - imagine the setting in which a learner places their entire budget on the two arms either side of the threshold. We show that in the problem dependent regime a binary search with corrections can match this bound, up to a $\log(K)$ term which disappears for large $T$. The intuition behind this is that as the depth of the tree is only $\log(K)$ the binary search can quickly find the point of interest and spend the majority of its time there. As both the concave and monotone problems can be solved with a binary search they therefore have the same rate. 

In the problem independent regime the situation is slightly more nuanced. In terms of lower bounds one is no longer restricted to a narrow class of problems and can consider a number of different problems, all close in terms of distributional distance but nevertheless disagreeing on the classification of certain arms above or below threshold. The cardinality of these sets differs between the monotone and concave setting - being $\log(K)$ and $\loglog(K)$ respectively. This then leads to a difference in the lower bound.  Upper bounds naturally must follow suit, while an adaptation of the standard binary search is still optimal in the monotone case in the concave case an algorithm using a binary search on a log scale is required. The above is by no means a rigorous explanation but hopefully gives the reader some intuition behind the differences in rates between the problem dependent and independent regimes, for more detail refer to \citet{cheshire2020}.\vspace{-0.3cm}

\section*{Acknowledgements}
The work of J.~Cheshire is supported by the Deutsche Forschungsgemeinschaft (DFG) DFG - 314838170, GRK 2297 MathCoRe.
The work of P.~M\'enard is supported by the SFI Sachsen-Anhalt for the project RE-BCI ZS/2019/10/102024 by the Investitionsbank Sachsen-Anhalt. The work of A. Carpentier is partially supported by the Deutsche Forschungsgemeinschaft (DFG) Emmy Noether grant MuSyAD (CA 1488/1-1), by the DFG - 314838170, GRK 2297 MathCoRe, by the DFG GRK 2433 DAEDALUS (384950143/GRK2433), by the DFG CRC 1294 'Data Assimilation', Project A03, and by the UFA-DFH through the French-German Doktorandenkolleg CDFA 01-18 and by the UFA-DFH through the French-German Doktorandenkolleg CDFA 01-18 and by the SFI Sachsen-Anhalt for the project RE-BCI.
\bibliographystyle{icml2021}
\bibliography{biblio}

\newpage
\appendix
\onecolumn 
\newpage
\section{Related work}\label{app:related}

\paragraph{Unstructured \tbp}
As mentioned in Section \ref{sec:minmaxthm} and demonstrated in Table \ref{tab:PD1} the unstructured problem dependent \tbp is already well studied in the literature, see \cite{chen2014combinatorial, chen2016combinatorial} for the fixed confidence setting and \citet{chen2014combinatorial, locatelli2016optimal, mukherjee2017thresholding}, \citet{JieZhong17bandits} for the fixed budget. As mentioned in Section \ref{sec:introduction}, \cite{locatelli2016optimal} is most relevant to our setting as they consider the fixed budget case. Their rate for the unstructured case depends upon the distribution of gaps across all the arms, which is of course to be expected. This again highlights the fact that the rate for the monotone setting depends only upon the minimum gap - that is the one adjacent to the threshold. 
%\al{Here maybe we should cite more stuff - eg there are tons of literature in the fixed confidence setting. This is the place to cite this.}

\paragraph{Monotone constraint \stbp} The \stbp problem was first introduced by \citet{garivier2017thresholding} in the context of drug dosing. Their results are in contrast to ours as they consider the \textit{fixed confidence} setting. Furthermore the algorithm proposed is shown to be optimal only in the asymptotic case, i.e when the confidence $1-\delta$ converges to $1$. The monotone shape constraint of the \stbp implies it is related to a noisy binary search i.e.~inserting an element into its correct place within an ordered list when only noisy labels of the elements are observed. A naive approach to the \stbp would be a binary search with $\frac{T}{\log(K)}$ samples at each step of the binary search. However for our setting this is not optimal, even in the problem independent case, see \cite{cheshire2020}. In \cite{feige1994computing} this issue is solved by introducing a binary search with corrections. They describe an algorithm structurally similar to \explore, using a binary tree with infinite extension however they consider a simpler setting where the probability of correct labeling is fixed as some $\delta > \frac{1}{2}$ and go on to show that there exists an algorithm that will correctly insert an element with probability at least $1-\delta$ in $\mathcal{O}\left(\log\left(\frac{K}{\delta}\right)\right)$ steps. For further literature on the related yet different problem of noisy binary search see, \cite{feige1994computing}, \citet{ben2008bayesian}, \citet{emamjomeh2016deterministic}, \cite{Nowak09binary}. Again, these papers consider settings with more structural assumptions than our own and are focused on the problem independent, fixed confidence regime. The minimax rate on expected regret for the problem independent \MTB is presented by \citet{cheshire2020}.

For us the adaptation of the algorithm in Cheshire20 to the problem dependent regime is not obvious. An important fact in our problem dependent regime is that the number of arms $K$ stops appearing in the error bound which is of order $\exp(-c T\Delta_{\min}^2)$ whenever T is large enough, i.e. larger than $\log(K)/\Delta_{\min}^2$. In \citet{cheshire2020}, the number of arms appeared in all bounds and was the main topic of study therein - the bound for the monotone problem was $\sqrt{\log(K)/T}$. A key interesting phenomenon here is that somewhere between the problem independent and problem dependent regime, $K$ stops playing a role. This implies that a very different dynamic is happening in the problem dependent regime, as compared to the problem independent regime.

Precisely in \citet{cheshire2020} they consider a sequence of events $\xi_t$ that depend on $K$ and occur with constant probability - which is the target probability of error in the worst case. Lemma 15 therin then applies Hoeffding's-Azuma to the summation of the indicator functions of said events to achieve a bound on the probability of making too many bad decisions in the tree. In order to achieve a problem dependent bound, we consider events $\xi_t$ which are problem dependent - they depend on $\Delta_{\min}$ - but NOT on $K$. This event is now problem dependent and the probability of its complement depends on both $\Delta_{\min}$ and $T/\log(K)$ (the number of times we sample each arm), i.e. is of order $\exp(-c T\Delta_{\min}/\log K)$, which, interestingly, is NOT the target probability of error in the problem dependent regime, but is quite larger. Our Lemma 22 is then substantially more than just a problem dependent adaptation of Lemma 15 of \citet{cheshire2020}, as we need to leverage the fact that there are many events $\xi_t$ - here $\log K$ - in order to bypass the fact that the probability of each individual $\xi_t$ depends on $K$ in our setting. We use a Chernoff bound to bound the sum of the indicator functions of said events  - and then in turn the probability of error - by $\exp(-T \Delta_{\min}^2)$ - which is much smaller than the probability of each individual $\xi_t$. This phenomenon is not needed in \citet{cheshire2020}.

Another point in favour of the PD-MTB is that it is significantly simpler than that of the MTB of \citet{cheshire2020}. We use an infinite extension to the binary tree which allows it to take the final node as output. This means we don’t require an additional subroutine to choose from a list of arms the algorithm has collected.

\paragraph{Concave constraint} 
To the best of our knowledge the \ctbp was first introduced in \cite{cheshire2020} in the \textit{problem independent regime}. However the related problems of estimating a concave function and optimising a concave function are well studied in the literature. Both problems are considered primarily in the continuous regime which makes comparison to the $K$-armed bandit setting difficult. The problem of estimating a concave function has been thoroughly studied in the noiseless setting, and also in the noisy setting, see e.g.~\cite{simchowitz2018adaptive}, where a continuous set of arms is considered, under H\"older smoothness assumptions. The problem of optimising a convex function in noise without access to its derivative - namely zeroth order noisy optimisation - has also been extensively studied. See e.g.~\cite{yudin}[Chapter 9], and~\cite{wang2017stochastic, agarwal2011stochastic, liang2014zeroth}
to name a few, all of them in a continuous setting with dimension $d$. The
focus of this literature is however very different to ours and \cite{cheshire2020}, as the main difficulty under their assumption is to obtain a good dependence in the dimension $d$, and with this in mind logarithmic factors are not very relevant.

% To compare this result to our own first note that roughly speaking one will need to bootstrap \al{what do you mean by this?} with $\frac{\log(\frac{1}{\delta})}{\Dm^2}$ at each stage of the binary search to give a probability greater than $\delta$ of correct labelling at each step. This then gives the result of \cite{feige1994computing} as $\mathcal{O}\left(\frac{1}{\Dm^2}\log\left(\frac{K}{\delta}\right)\right)$ when transferred to our setting. Now let us extend our result to the fixed confidence setting, assuming \explore takes $\delta$ as a parameter it is an immediate corollary of Theorem \ref{thm:depmon_up} that the \explore algorithm will correctly classify the all arms with probability at least $1-\delta$ in $\mathcal{O}\left(\frac{1}{\Dm^2}\log\left(\frac{1}{\delta}\right)\right)$ time steps. 

%\al{This should be smoothed and precised a bit - I don't really understand what you mean in some places. Also I guess that we should cite more monotone papers - the ones of the other paper. And concave is missing.}

% !TEX root = ../main.tex
\section{Potential further work: Algorithms that are problem dependent and minimax-optimal simultaneously: Unimodal shape constraint} 

As described earlier after the related theorems, our algorithm \explore is optimal for minimising the probability of error, in a problem dependent sense, and up to universal multiplicative constants in the exponential. A relevant question is on whether it is possible to construct a strategy that is optimal both in this problem dependent sense, but also in a problem independent sense - i.e.~global minimax - when it comes to the simple regret.

 While designed for the problem independent regime - and reaching in this regime the minimax optimal simple regret of order $\sqrt{\frac{\log K}{T}}$ - we conjecture the \MTB algorithm, described by \citet{cheshire2020} is optimal also in the problem dependent regime, i.e.~that it achieves an upper bound on the probability of error of same order as that of \explore in Theorem~\ref{thm:depmon_up}. However note that to prove such an opitmaility, at least for us, would be none trivial, see the above Section \ref{app:related}.

 As with \explore the \MTB algorithm takes a monotone bandit problem mapped to a binary tree - although without the infinite extension, as input. The \MTB algorithm then consists of two sub algorithm. The first, \piexplore is an exploration phase, identical to our algorithm \explore. However, as opposed to simply outputting the end node the history of the random walk is passed to the second algorithm, \choose. The algorithm \choose selects all arms whose sample mean is within a certain tolerance of the threshold - chosen to be as small as possible while still producing a none empty set, and then takes the median of said set. This additional step is required as the \MTB algorithm aims to achieve the minimax rate on \textit{expected regret} - that is $\sqrt{\frac{\log (K)}{T}}$, and therefore wishes to output any arm $k : |\mu_k - \tau|\lesssim \sqrt{\frac{\log (K)}{T}}$. The idea being that during the explore phase enough time will be spent on nodes containing such arms. 
 
 If we consider the problem dependent regime, and whenever we are not in the trivial regime where $\Dm \lesssim \sqrt{\frac{\log (K)}{T}}$, we conjecture that the \MTB algorithm will spend sufficient time on the unique node $\tilde{v} : \mu_{\tilde{v}(l)} <\tau<\mu_{\tilde{v}(r)}$ with high probability matching the bound of Theorem~\ref{thm:depmon_up}. The algorithm \choose will then output arm $\tilde{v}(r)$. The problem dependent regime allows for a less convoluted approach - indeed \explore is very simple in comparison to \MTB. However, it is nevertheless important to note that for the monotone setting there exists an algorithm that is optimal in both problem dependent and problem independent regimes. 
 
 In regards to the concave case it is not as immediate that the \ctb algorithm by \citet{cheshire2020} will also be optimal in the problem dependent concave setting. The \ctb algorithm is significantly more complex than the \MTB as it successively applies a noisy binary search on a log scale to find arms increasingly close to threshold at a geometric rate. We however conjecture that it will be the case the \ctb is also optimal in the problem dependent regime. 
\subsection{Unimodal constraint}
A natural additional shape constraint for the $\tbp$ is a Unimodal one. Indeed bandit problems with a unimodal constraint are already considered in the literature, for the problem of minimising the cumulative regret or identifying the best arm under unimodal constraints see \citet{yu2011unimodal}, \citet{combes2014unimodal2}, \citet{paladino2017unimodal} and \citet{combes2014unimodal}. The \tbp in particular with a unimodal constraint is studied in \citet{cheshire2020} in the \emph{problem independent} regime. With the above work already in hand it is natural to consider a unimodal shape constraint on the $\tbp$ in the \emph{problem dependent} regime. A possible algorithm would be one which, similar to the \CTB, first finds an arm above threshold and then reduces the problem to one with a monotone constraint. We conjecture that if one considers a class of problems with $M$ arms above threshold the regret of the problem will be dominated by that of finding a single arm above threshold and will be of the order $\exp\left(-\frac{MT\Dm}{K} \right)$ with a matching lower bound. If one wishes to consider a narrower class based on a single vector of gaps, as in the concave or monotone setting one might hope to achieve a rate $\exp\left(- \frac{MT}{K}(\frac{1}{M}\sum_{i=1}^M \bD_i)^2\right)$ however this result, for both an upper and lower bound, appears not so straightforward. 
% !TEX root = ../main.tex
\section{Proofs relating to the Monotone setting}
\label{app:proof_monotone}
We first state a useful inequality. Let $\kl(p,q)$ be the Kullback-Leibler divergence between two Bernoulli distributions of parameter $p$ and $q$, 
\[
\kl(p,q) = p\log\left(\frac{p}{q}\right)+ (1-p) \log\left(\frac{1-p}{1-q}\right)
\]
It holds 
\begin{align}
    \kl(p,q) &=p\log\left(\frac{1}{q}\right) + (1-p)\log\left(\frac{1}{1-q}\right) +p\log(p)+(1-p)\log(1-p) \nonumber  \\
    &\geq p\log\left(\frac{1}{q}\right) -\log(2)\label{eq:fano}\,.
\end{align}
\begin{proof}[Proof of Theorem \ref{thm:depmon_down}]
\label{proof:depmon_down}
We denote by $N_k^t$ the number of times the arm $k$ is pulled until and included time $t$, i.e. $N_k^t = \sum_{s=1}^t \ind_{k_s=k}$.
Let $i = \argmin_{k\in[K]} \bD_k$, that is $\bD_i = \Dm$. Consider the two bandit problems $\unu^+$ and $\unu^-$ where 
\begin{align*}
    \nu^+_k =\begin{cases} \cN(\bD_k,\sigma^2) &\text{ if } k\geq i\\
    \cN(-\bD_k,\sigma^2) &\text{ else}
    \end{cases}\,, \qquad
    \nu^-_k =\begin{cases} \cN(\bD_k,\sigma^2) &\text{ if } k> i\\
    \cN(-\bD_k,\sigma^2) &\text{ else}
    \end{cases}\,.
\end{align*}
Note these bandit problems belong to the class of \stbp $\Bsp$. In particular we can lower bound the error by the probability to make a mistake in the prediction of the label of arm~$i$
\[
e_T^{\unu^+} \geq \PP_{\unu^+}(\hQ_i = -1) \qquad e_T^{\unu^-} \geq \PP_{\unu^-}(\hQ_i = 1)\,.
\]
We can assume that $\PP_{\unu^+}(\hQ_i = -1) \leq 1/2$ otherwise the bound is trivially true. Thanks to the chain rule then the contraction of the Kullback-Leibler divergence (e.g. see \citet{garivier2019explore}) and \eqref{eq:fano}, it holds 
\begin{align*}
    T\frac{\Dm^2}{2\sigma^2}\geq \EE_{\unu^+} [N_i^T] \frac{\Dm^2}{2\sigma^2} &= \KL(\PP_{\unu^+}^{I_T}, \PP_{\unu^-}^{I_T})\\
    & \geq \kl\big(\PP_{\unu^+}(\hQ_i = 1),\PP_{\unu^-}(\hQ_i = 1)\big)\\
    &\geq \PP_{\unu^+}(\hQ_i = 1) \log\!\!\left(\frac{1}{\PP_{\unu^-}(\hQ_i = 1)}\right) -\log(2)\,,
\end{align*}
where we denote by $\PP_{\unu}^{I_T}$ the probability distribution of the history $I_T$ under the bandit problem $\unu$. Thus, using that $\PP_{\unu^+}(\hQ_i = 1) = 1-\PP_{\unu^+}(\hQ_i = -1) \geq 1/2$ we obtain 
\[
\PP_{\unu^-}(\hQ_i = 1) \geq \frac{1}{4}\exp\!\left(-\frac{T\Dm^2}{\sigma^2}\right)\,.
\]
Which allows us to conclude that 
\[
\max(e_T^{\unu^+},e_T^{\unu^-}) \geq \frac{1}{4}\exp\!\left(-\frac{T\Dm^2}{\sigma^2}\right)\,.
\]
% Define $\cA_i$ as the event under which the learner classifies the $i$th arm above threshold, that is $\cA_i := \{\hQ[i] = 1\}$. We now consider the following event,  

% % \[\xi := \left\{\forall k<K,  \sup_{t<T}\left|\sum_1^t X_i - t\Delta \right| \leq \sqrt{T\log(K)}\right\}\]

% \[\xi_i := \left\{ \sup_{t<T}\left|\sum_{j=1}^t X_{i,t} - t\bD_i \right| \leq \sqrt{T}\right\}\]

% which via a hoeffding martingale inequality occurs with probability greater than $\frac{3}{4}$, and proceed with the change of measure, 
 
% \begin{align*}
% \bP_{\nu_{\bD_i^+}}\left(\cA_i\right) &= \bE_{\nu_{\bD_i^-}}\left(\ind_{\cA_i} \exp \left(-2\sum_{t=1}^{T_i} X_{i,t}\frac{\Dm}{\sigma^2} \right)\right)\\ &\geq\bE_{\nu_{\bD_i^-}}\big(\ind_{\cA_i \cap \xi} \exp \left(-2\frac{T_i\Dm^2}{\sigma^2} - \Dm\sqrt{T} \big)\right)\\
% &\geq \bP_{\nu_{\bD_i^-}}\big(\cA_i\cap \xi\big) \exp \left(-3\frac{T\Dm^2}{\sigma^2}\right)
% \end{align*}
% By considering the two cases: $\bP_{\nu_{\bD_i^+}}\big(\ind_{\cA_i}\big) < \frac{1}{2}$ and $\bP_{\nu_{\bD_i^+}}\big(\ind_{\cA_i}\big) > \frac{1}{2}$ we have that 

% \begin{align*}
% R_{T}^{\nu,\pi} &\geq \min\left(\frac{1}{2},\frac{1}{2}\exp \left(-\frac{T\Dm^2}{\sigma^2}\right)\right)\\
% &\geq \min\left(\frac{1}{2},\frac{1}{2}\exp \left(-\frac{T\Dm^2}{2\sigma^2}\right)\right)
% \end{align*}
% for some $\nu \in \{\nu_{\bD_i}^+,\nu_{\bD_i}^-\}$ 
\end{proof}

\paragraph{Proof of Theorem \ref{thm:depmon_up}.}\label{proof:depmon_up} We assume in the proof, without loss of generality, that \[
\Dmin \geq \cmin \sqrt{\frac{\sigma^2\log(K)}{T}}
\]
with $\cmin = 13$. Indeed, otherwise, the bound of Theorem~\ref{thm:depmon_up} is trivially true. 

The proof of Theorem \ref{thm:depmon_up} is structured in the following manner. In our \emph{original} binary tree we know there is a unique leaf $v_{\Delta}$, such that $\tau \in [\mu_{v_{\Delta} (l)},\mu_{v_{\Delta} (r)}]$. Essentially we want to show that the explore algorithm will terminate in the subtree of this $v_{\bD}$ with high probability - recall that we extend our binary tree by attaching an infinite sub tree to each leaf, the nodes of which are identical to the respective leaf. At time $t$ we say our algorithm makes a favourable decision if all sample means are well concentrated - that is with $\Dmin$ of their true mean. On such a favourable decision we show that the explore algorithm will make a step towards the subtree of $v_\Delta$, or go deeper if it is already in it. Therefore if overall we can make sufficient proportion of favourable events we are guaranteed to terminate in the subtree of $v_{\Delta}$. We then show that this favorable event holds with high probability.
% For the proof of Theorem \ref{thm:depmon_up} we consider a more general set of problems, given $\epsilon > 0$, define, 

% $$\Bs^{*,\epsilon} := \{\cB : \left( \min(|\mu_i - \tau|,\epsilon)\sign(\mu_i - \tau)\right)_{k \leq K} \; \mathrm{is\;an\; increasing\;sequence} \}\;. $$

% Note that for all $\epsilon > \Dm$, $\Bsp \subset \Bs^{*,\epsilon}$, hence all results will hold also in the unaltered monotone setting. In the \textit{problem dependent} regime, for $\bD \in \DeltaBs$, we consider the following class of problems
% \[\Bsep = \{\nu \in \Bs^{*,\epsilon} : \forall k \in [K], \; |\mu_k - \tau|= \bD_k\}\;.\]

\paragraph{Step 1: Initial definitions and lemmas}
We denote by $ST(v)$ the subtree rooted at node $v$.
\begin{definition}\label{def:subtree}
The subtree $ST(v)$ of a node $v$ is defined recursively as follows: $v \in ST(v)$ and
$$\forall \; q \in ST(v), \; L(q), R(q) \in ST(v)\;.$$
\end{definition} 
We define $Z_{\Dmin}$, the set of good nodes, as
\begin{equation*}%\label{eq:z1}
Z_{\Dmin} := \{ v : \exists k \in \{l,m,r\} : |\mu_{v(k)} - \tau| \leq \Dmin\}\;,
\end{equation*} 
Note that $Z_{\Dmin}$ is simply the leaf $v_{\Delta}$ and it's sub tree attached during the infinite extension of the binary tree. 
At time $t$ we define $w_t$ as the node of maximum depth whose subtree contains both $v_t$ and  $Z_{\Dmin}$. Formally, for $t \leq T_1$, we let
\begin{equation}\label{eq:w}
w_t \in \argmax_{\{v: \tau \in [\mu_{v(l)},\mu_{v(r)}] \; \& \;v_t \in ST(v)\}} |v|\,.
\end{equation} 

\begin{lemma}\label{lem:w_unique}
The node $w_t$ is unique.
\end{lemma}
\begin{proof} At time $t$ consider, a node $q_t$ which also satisfies \eqref{eq:w}. As $v_t \in ST(w_t)$ and $v_t \in ST(q_t)$ we can assume without loss of generality $q_t \in ST(w_t)$ with $|q_t| \geq |w_t|$. This then implies, from \eqref{eq:w}, that $|q_t| = |w_t|$ and as $q_t \in ST(w_t)$, we have $q_t =  w_t$.
\end{proof}

For $t\leq T_1$ we define $D_t$ as the relative distance from $v_t$ to $v_{\Delta}$, it is taken as the length of the path running from $v_t$ up to $w_t$ and then down (or up if $v_t \in Z_{\Dmin}$) to $v_{\Delta}$. Formally, we have
\[D_t := \left|v_t\right| - \left|w_t\right| + |v_{\Delta}|-|w_t| . \]
Note the following properties of $D_t$ and $w_t$,
\begin{gather}
    ST(v_t) \cap Z_{\Dmin} \neq \emptyset \Rightarrow v_t = w_t \;,\label{eq:ST} \\
    D_t \leq 0 \Rightarrow v_t = w_t \; \text{and} \; w_t,v_t \in Z_{\Dmin} \;. \label{eq:D0}
\end{gather}
We define the favorable event where the estimates of the means are close to the true ones for all the arms in $v_t$, At time $t$ we define the event 
$$\xi_t := \{\forall k\in\{l,m,r\}, |\hmu_{k,t} - \mu_{v_t(k)}| \leq \Dmin\}$$
and we denote $\bar{\xi}_t$ as the complement of $\xi_t$.

\paragraph{Step 2: Actions of the algorithm on all iterations}
After any execution of algorithm \explore note the following, for $t \leq T_1$,\; $v_t$ and $v_{t+1}$ are separated by at most one edge, i.e. 
\begin{equation}\label{eq:1edge}
v_{t+1} \in \{L(v_t), R(v_t), P(v_t) \}\,.
\end{equation} 
\begin{lemma}\label{lem:allit_mono} 
On execution of algorithm \explore for all $t \leq T_1$ we have the following,
\begin{equation*}
    D_{t+1} \leq D_t + 1 %\label{eq:dxi},
\end{equation*}

%$$D_{t+1}^\epsilon \leq D_t^\epsilon + 1\; \mathrm{and} \; G_{t+1}^\epsilon \geq G_t^\epsilon \;.$$
\end{lemma}
\begin{proof}
As the algorithm moves at most 1 step per iteration, see~\eqref{eq:1edge}, for $t \leq T_1$,  it holds
$$\left|v_{t}\right| - \left|w_{t}\right| \geq \left| v_{t+1}\right| - \left|w_{t} \right| -1\;.$$
We consider two cases. Firstly, assume we are in the event $\{v_{t+1} \neq P(v_t)\} \cup \{w_t \neq v_t\}$. Under this event note that $v_{t+1} \in ST(w_t)$. It follows
\begin{align*}
D_t &= \left|v_{t}\right| - \left|w_{t}\right|  + |v_{\Delta}|-|w_t| \\
& \geq \left|v_{t+1}\right| - \left|w_{t}\right|  + |v_{\Delta}|-|w_t| -1\\
& \geq  \left|v_{t+1}\right| - \left|w_{t+1} \right|  +  |v_{\Delta}| - |w_{t+1}| -1\\
&= D_{t+1} - 1\;,
\end{align*}
where the third line comes from the definition of $w_{t+1}$, see~\eqref{eq:w}.\\
In the case where $w_t = v_t$ and $v_{t+1} = P(v_t)$ note that $w_{t+1} = v_{t+1}$ and, 
\[D_{t+1} = \left|v_{\Delta}\right| - \left|w_{t+1}\right| = \left|v_{\Delta}\right| - \left|w_{t}\right| + 1 = D_t + 1\;.\]
Therefore in all cases we have $D_{t+1} \leq D_t + 1$.
\end{proof} 

\paragraph{Step 3: Actions of the algorithm on $\xi_t$}
\begin{lemma}\label{lem:dxic_mono}
On execution of algorithm \explore for all $t \leq T_1$, on $\xi_t$, we have the following,
\begin{equation*}
    D_{t+1} \leq D_t- 1\;.
\end{equation*}%\label{eq:gxic} 
\end{lemma}
\begin{proof} Note that on the favorable event $\xi_t$, we have $\forall j \in \{l,m,r\}$,
\begin{gather}\label{eq:step3_mono1}
\mu_{v_t(j)} \geq \tau \Rightarrow \hat\mu_{j,t} \geq \tau\,, \\
\mu_{v_t(j)} \leq \tau \Rightarrow \hat\mu_{j,t} \leq \tau \, \label{eq:step3_mono2}.
\end{gather} 
We consider the following three cases:
\begin{itemize} 

\item If $\tau \notin \left[\mu_{v_{t}(l)},\, \mu_{v_{t}(r)} \right]$. From~\eqref{eq:step3_mono1} and~\eqref{eq:step3_mono2}, under $\xi_{t}$, we get $\tau \notin \left[\hat{\mu}_{l,t},\hat{\mu}_{r,t}\right]$, and therefore $v_{t+1} = P(v_t)$ and $w_t = w_{t+1}$. Thus thanks to Lemma~\ref{lem:w_unique}, under $\xi_t$,
$$  D_{t+1} = \left|v_{t+1}\right| - \left|w_{t+1}\right| + |v_{\Delta}|-|w_{t+1}|  = \left|v_{t}\right| -1 - \left|w_{t}\right| +|v_{\Delta}|-|w_{t}| = D_t - 1 \,.$$

\item If $\tau \in \left[\mu_{v_{t}(l)},\,\mu_{v_{t}(r)}\right]$ and $v_t \notin Z_{\Dmin}$. Note that in this case $v_t$ can not be a leaf and we just need to go down in the subtree of $v_t$ to find $v_{\Delta}$, id est $w_t = v_t$. Since $v_t\notin Z_{\Dmin}$, without loss of generality, we can assume for example $\mu_{v_t(m)} > \tau$. From~\eqref{eq:step3_mono1} and~\eqref{eq:step3_mono2}, under $\xi$, we then have $\tau\in [\hat{\mu}_{l,t},\,\hat{\mu}_{r,t}]$ and $\hat{\mu}_{m,t} \geq \tau$. Hence algorithm \explore goes to the correct subtree, $v_{t+1} = L(v_t)$. In particular we also have for this node 
$$ \tau \in \left[\mu_{v_{t+1}(l)},\,\mu_{v_{t}(m)}\right] \,,$$
therefore it holds again $w_{t+1} = v_{t+1}$. Thus combining the previous remarks we obtain thanks to Lemma~\ref{lem:w_unique}, under $\xi_t$,
$$ D_{t+1} = |v_{\Delta}| - |w_{t+1}| = |v_{\Delta}|-|w_t|-1 =  D_t - 1 \;.$$

\item If $\tau \in \left[\mu_{v_{t}(l)},\,\mu_{v_{t}(r)} \right]$ and $v_t \in Z_{\Dmin}$. Firstly note that $w_t = v_t$. Now, using the same reasoning as in the previous case, as $\tau \in [\mu_{v_{t}(l)},\,\mu_{v_{t}(r)}]$ we have $v_{t+1} = L(v_t)$ or $v_{t+1} = R(v_t)$. In either case we get $v_{t+1} \in Z_{\Dmin}$ because of \eqref{eq:step3_mono1} and \eqref{eq:step3_mono2}, thus it holds $w_{t+1} = v_{t+1}$. Therefore we have 
$$  D_{t+1} = \left|v_{t+1}\right| - \left|w_{t+1}\right| + |v_{\Delta}|-|w_{t+1}|  =  |v_{\Delta}|-|w_{t}| - 1 = D_t - 1 \,.$$
\end{itemize}

\end{proof}

\paragraph{Step 4: Upper bound on $D_{T_1+1}$}

\begin{lemma}\label{lem:ub_on_Dt} For any execution of algorithm \explore 
\[
D_{T_1+1} \leq 2\sum_{t=1}^{T_1}\ind_{\bxi_t} - \frac{3T_1}{4}\,.
\]
\end{lemma}
\begin{proof}
Combining Lemma~\ref{lem:allit_mono} and Lemma~\ref{lem:dxic_mono} respectively we have
\begin{equation*}
D_{t+1} \leq D_t+\ind_{\bxi_t}-\ind_{\xi_t}\,.
\end{equation*}
Using this inequality we obtain 
\begin{align*}
D_{T_1+1} &= D_1 + \sum_{t=1}^{T_1} \left(D_{t+1}-D_{t}\right) \\
&\leq D_1 + \sum_{t=1}^{T_1} \big(\ind_{\bxi_t}-\ind_{\xi_t}\big)\\
&\leq D_1+2\sum_{t=1}^{T_1}\ind_{\bxi_{t}} - T_1\\
&\leq 2\sum_{t=1}^{T_1}\ind_{\bxi_t} - \frac{3T_1}{4}\,,
\end{align*}
where we used in the last inequality the fact that $D_1\leq \log_2(K)$ and that $\log_2(K) \leq T_1/4$ by definition of $T_1$ .
\end{proof}

\begin{lemma}\label{lem:chern} 
For $\cmon = 1/48$ and $\cmon' = 12$ it holds 
\[
\PP\left( \sum_{t=1}^{T_1} \ind_{\bar{\xi}_t} \geq \frac{T_1}{4}\right) \leq \exp\!\left(\cmon\frac{-T\Dmin^2}{\;\sigma^2}\right)\,.
\]
\end{lemma}
\begin{proof}
Let $\cF_t$ be the information available at and including step $t$ of algorithm \explore. Thanks to the Chernoff inequality and the choice of $T_2$, we have for all $j\in \{l,m,r\}$,
\begin{align*} 
\mathbb{P}\left( \left|\hat{\mu}_{j,t} - \mu_{v_t(j)}\right| \geq  \Dmin |\cF_{t-1} \right) &\leq 2\exp\!\left(-\frac{T_2\Dmin^2}{2\sigma^2}\right)\\
&\leq 2\exp\!\left(-\cmin^2 \frac{\log(K)}{36\log(K)+6}\right)\\
&\leq \frac{1}{24} 
\end{align*} 
as we assume $ \Dmin > \cmin\sqrt{\frac{\sigma^2 \log K}{T}}$ and $\cmin \geq 13$. Therefore by a union bound \begin{equation}
\label{eq:pt_upper_bound}
p_t := \PP(\bar{\xi}_t|\cF_{t-1})\leq p_0:=6\exp\!\left(-\frac{T_2\Dmin^2}{2\sigma^2}\right) \leq \frac{1}{8}.
\end{equation}

We will apply the Chernoff inequality to upper bound the sum of indicator function.  Thanks to the Markov inequality for $\lambda\geq 0$ we have 
\begin{align}
    \PP\left(\sum_{t=1}^{T_1} \ind_{\bar{\xi}_t} \geq \frac{T_1}{4}\right)
    &\leq \EE\left[\exp\!\left(\lambda\sum_{t=1}^{T_1} \ind_{\bar{\xi}_t}\right)\right]e^{-\lambda\frac{T_1}{4}}\,.\label{eq:markov_mono}
\end{align}
Let $\phi_p(\lambda) = \log(1-p+p e^\lambda)$ be the log-partition function of a Bernoulli of parameter $p\in [0,1]$. Note that for $\lambda\geq 0$, since $p\mapsto \phi_p(\lambda)$ is non-decreasing and because of \eqref{eq:pt_upper_bound} it holds $\phi_{p_t}(\lambda) \leq  \phi_{p_0}(\lambda)$ for all $t$. Thus by induction we have 
\begin{align*}
\EE\left[\exp\!\left(\lambda\sum_{t=1}^{T_1} \ind_{\bar{\xi}_t}\right)\right] &= \EE\left[\EE\left[\exp\!\left(\lambda\ind_{\bar{\xi}_t}\right)|\cF_{T_1-1}\right]\exp\!\left(\lambda\sum_{t=1}^{T_1-1} \ind_{\bar{\xi}_s}\right)\right]\\
&= \EE\left[ e^{\phi_{p_{T_1}}(\lambda)}\exp\!\left(\lambda\sum_{t=1}^{T_1-1} \ind_{\bar{\xi}_t}\right)\right] \leq e^{\phi_{p_0}(\lambda)} \EE\left[ \exp\!\left(\lambda\sum_{s=1}^{t-1} \ind_{\bar{\xi}_t}\right)\right]\\
&\leq \EE\left[e^{T_1 \phi_{p_0}(\lambda)}\right]\,.
\end{align*}
Then going back to \eqref{eq:markov_mono}  and using that $\sup_{\lambda\geq 0} \lambda q -\phi_p(\lambda) = \kl(q,p)$ when $q\geq p$ we get 
\begin{align*}
    \PP\left(\sum_{t=1}^{T_1} \ind_{\bar{\xi}_t} \geq \frac{T_1}{4}\right) \leq \exp\!\left(-T_1\sup_{\lambda\geq 0}\left(\lambda \frac{1}{4} - \phi_{p_0}(\lambda)\right)\right)=e^{-T_1 \kl(1/4,p_0)}\,.
\end{align*}
It remains to conclude with \eqref{eq:fano}
\begin{align*}
    T_1\kl(1/4,p_0) &\geq T_1\frac{1}{4}\log\left(1/p_0\right)-T_1 \log(2)\\
    &\geq \frac{1}{8} \frac{T_1 T_2 \bD_{\min}^2}{\sigma^2}-T_1\left(\log(2)+\frac{1}{4}\log(3)\right)\\
    &\geq \frac{1}{48} \frac{T \bD_{\min}^2}{\sigma^2}-T_1\left( \log(2)+\frac{1}{4}\log(3)\right)\\
    &\geq \cmon  \frac{T \bD_{\min}^2}{\sigma^2} -\cmon'\log(K)
\end{align*}
where $\cmon = 1/48$ and $\cmon' = 12$.

% \pierre{
% Using the previous inequalities one can prove by induction that 

% By application of (chernoff lemma), we have, 
% \begin{align*}
% \mathbb{P}\left( \sum_{t=1}^{T_1} \left[\ind_{\xi_t^{c}} - \PP(\xi_t^{c}|\cF_{t-1}) \right]\geq  \frac{T_1}{8} \right) &\leq \exp\left(-\cbin \frac{T_1}{8} \log\left(1 + \frac{\exp\left(\frac{T_2\Dm^2}{2\sigma^2}\right)}{12}\right) \right)\\
% &\leq \exp\left(-\cmon \frac{T_1 T_2 \Dm^2}{\sigma^2}\right)\\
% &= \exp\left(-\cmon \frac{T\Dm^2}{\sigma^2} \right)
% \end{align*} 
% where the second line comes from a suitable choice of constant $\cmin$ such that $T_2 \Dm^2$ is sufficiently large. }

\end{proof}

By combination of Lemmas \ref{lem:ub_on_Dt} and \ref{lem:chern} we have that $D_{T_1+1} \leq 0$ with probability greater than $1 -\exp\left(-\cmon \frac{T\Dmin^2}{\sigma^2} + \cmon'\log(K)\right)$. Thus with said probability we output an arm $\hk$ such that  $ \tau \in [\mu_{\hk},\mu_{\hk+1}]$.

% !TEX root = ../main.tex
\section{Proofs relating to the concave setting}
\label{app:proof_concave}

Before proceeding with the proof of Theorem \ref{thm:depcon_down} we present the following structural lemma.
\begin{lemma}
\label{lem:alternative_lb_ccv}
Let $\Delta \in \DeltaBc$ and let $(\mu_k)_k$ be an associated concave sequence of means. There exists a sequence of means $(\mu'_k)_k$ - with associated gaps $\Delta' = |\mu' - \tau|$ - such that
\begin{enumerate}
    \item[(a)] $(\mu'_k)_k$ is concave.
    \item[(b)] $\mu$ and $\mu'$ have not all the arms classified in the same way:
    \[\exists k \in [K] : \mathrm{sign} (\mu_k -\tau) \neq \mathrm{sign} (\mu_k -\tau) \,.\]
    \item[(c)] For all $k \in [K]$ it holds that
    $$|\mu_k' - \mu_k| \leq 3\Dmin.$$
    \item[(d)] For all $k \in [K]$ it holds that
    $$ \frac{\Delta_k}{10}\leq \Delta_k' \leq 3\Delta_k.$$
\end{enumerate}
%$\bD \in \DeltaBc$ and $\unu \in \Bsc$ such that $|\mu_k-\tau|=\bD_k$ for all $k\in[K]$ where $\mu$ is the vector of means associated with $\unu$. There exists a $\nu' \in \Bsc$ such that the following holds,
%\[\exists k \in [K] : Q^{\nu'}_k \neq Q^\nu_k\,.\]
\end{lemma}
\begin{proof}
%Let us write $k_L, k_R$ for the two arms that are `just' below threshold, i.e.~such that $\mu_{k_L} \leq \tau \leq \mu_{k_L+1}$ and $\mu_{k_R} \leq \tau \leq \mu_{k_R-1}$. These two arms can be defined without loss of generality since there is at least one arm above threshold, and since we can always take two virtual means $\mu_0$, $\mu_{K+1}$ at $-\infty$ on the boundaries.
Let $k^* \in\argmin_{k\in [K]} \Delta_k$. 
We proceed in two cases: either this arm is up threshold, or it is below threshold. In everything that follows we set $\Dmin := \min_{k\in[K]} \Delta_k$.%Since we know that there is at least one arm that is above threshold, we know

\paragraph{Case 1: Arm below threshold, i.e.~$\mu_{k^*} \leq \tau$.}

Let us write $k_L, k_R$ for the two arms that are `just' below threshold, i.e.~such that $\mu_{k_L} \leq \tau \leq \mu_{k_L+1}$ and $\mu_{k_R} \leq \tau \leq \mu_{k_R-1}$. These two arms can be defined without loss of generality since there is at least one arm above threshold, and since we can always take two virtual means $\mu_0$, $\mu_{K+1}$ at $-\infty$ on the boundaries. 

In the context where $\mu_{k^*} \leq \tau$ it is clear that we can pick $k^* \in \{k_L, k_R\}$ and so let us assume w.l.g.~that $k^* = k_L$.

In this case, we define $\mu'$ either:
\begin{itemize}
    \item if $\Delta_{k_R} \leq 3\Dmin/2$, for all $k\in[K]$, 
    $$\mu'_k = \mu_k + 2\Dmin.$$
    \item if $\Delta_{k_R} \geq 3\Dmin/2$, for all $k\in[K]$, 
    $$\mu'_k = \mu_k + 5\Dmin/4.$$
\end{itemize}
(a) holds as we just translated vertically the concave means. Also (b) holds since we switched the sign of arm $k^*$ by construction. (c) holds also  since we precisely added at most $2\Dmin$ to the means. And finally for (d): we have for any $k\in [K]$ that $|\bD_k - \Delta'_k| \leq 2 \Dmin$, so that
$$\Delta'_k \leq 3\Delta_k.$$
Moreover for all arms $k$ above threshold, it is clear that $\Delta_k' \geq \Delta_k$. On the other hand, for any arm $k$ below threshold and that are not next to an arm up threshold - i.e.~not $k_L$ or $k_R$ - we have by concavity that
$$\tau -\mu_{k} \geq 3\Dmin,$$
which implies
$$\tau -\mu'_{k} \geq \tau -\mu_{k} - 2\Dmin \geq \frac{\tau -\mu_{k}}{3},$$
i.e. $\Delta'_k \geq \Delta_k/3$. Finally for $\{k_L, k_R\}$: it is clear that $\Delta'_{k^*} \geq \Delta_{k^*}/4$ by construction so that $\Delta'_{k_L} \geq \Delta_{k_L}/4$. And also by construction:
\begin{itemize}
    \item if $\Delta_{k_R} \leq 3\Dmin/2$, then $\Delta'_{k_R} \geq \Dmin/2 \geq \Delta_{k_R}/3$.
    \item if $\Delta_{k_R} \geq 3\Dmin/2$, then $\Delta'_{k_R} \geq \Delta_{k_R} -  5\Dmin/4 \geq \Delta_{k_R}/6$.
\end{itemize}
So that in both situations (d) holds.

%since we know by definition of $\DeltaBc$ that at least one arm is above threshold, the arm either left or right from $k^*$ must be above threshold. Assume without loss of generality that the arm right to $k^*$ is above threshold, i.e.~that $\mu_{k^*+1} \geq \tau$. By concavity, we also know that $2\Dmin \leq \mu_{k^*+1} - \mu_{k^*} \leq \mu_{k^*} - \mu_{k^*-1}$, so that
%\tau -\mu_{k^*-1} \geq 3\Dmin.$$
%And so
%$$\tau -\mu'_{k^*-1} = \tau -\mu_{k^*-1} - 3\Dmin/2 \geq \frac{\tau %-\mu_{k^*-1}}{2}.$$

\paragraph{Case 2: Arm above threshold, i.e.~$\mu_{k^*} \geq \tau$.}

Note first that if $k^*$ is the only arm above threshold, we simply set for any $k$
$$\mu'_k = \mu_k - 2\Dmin,$$
and this satisfies the requirements (a)-(d). Assume now that this case does not hold, so that $k^* \in\{k_L+1, k_R-1\}$ and $k_L+1 < k_R - 1$.%w.l.o.g.~we have $\mu_{k^*+1} \geq \tau$ - and $k^* = k_L+1 < k_R - 1$. 
We now again consider several cases. Note that in any case $k^*\in \{k_L+1, k_R-1\}$.

\noindent
\textit{Sub-case 1: $\mu$ not too flat around the threshold.} Assume first that $\Delta_{k_L+2} \land \Delta_{k_R-2} \geq 3\Dmin/2$. Assume w.l.o.g.~that $k^* = k_L+1$. In this sub-case we define $\mu'$ either as:
\begin{itemize}
    \item if $\Delta_{k_R-1} \geq 5\Dmin/4$ set
$$\mu' = \mu - 9\Dmin/8,$$
\item otherwise if $\Delta_{k_R-1} \leq 5\Dmin/4$ set
$$\mu' = \mu - 11\Dmin/8.$$
\end{itemize}
It is clear that (a) holds (vertical translation of a concave sequence), (b) holds (arm $k^*$ changes sides of threshold) and (c) holds since we translate at most by $11\Dmin/8$. Now for (d): it is clear in both cases that $\Delta_k' \leq \Delta_k + 11\Dmin/8 \leq 3\Delta_k$. Moreover:
\begin{itemize}
    \item if $\Delta_{k_R-1} \geq 5\Dmin/4$, then for all $k \neq k^*$, we have $\Delta_k' \geq \Delta_k - 9\Dmin/8 \geq \Delta_k/8$ - and also by definition $\Delta_{k^*} = \Delta_{k^*}/8$. And so (d) holds in this case.
    \item if $\Delta_{k_R-1} \leq 5\Dmin/4$ we have for all $k$ such that $\mu_k \leq \tau$ that $\Delta_k' \geq \Delta_k$, and for any $k\in \{k_L+2, ..., k_R-2\}$ that $\Delta_k' \geq \Delta_k - 11\Dmin/8 \geq \Delta_k/8$ since for such $k$ we have $\Delta_k \geq 3\Dmin/2$. Also $\Delta_{k_L+1}' \geq \Delta_{k_R-1}' \geq \Dmin/8 \geq \Delta_{k_R-1}/10\geq \Delta_{k_L+1}/10$. And so (d) holds in this case.
\end{itemize}

\noindent
\textit{Sub-case 2: $\mu$ quite flat around the threshold.} Assume now that $\Delta_{k_L+2} \land \Delta_{k_R-2} \leq 3\Dmin/2$. Assume w.l.o.g.~that $\Delta_{k_L+2} \leq 3\Dmin/2$ and set
$$\mu'_{k_L+1} = \mu_{k_L+1} - 9\Delta_{k_L+1}/8.$$
and for $k\neq k_L+1$
$$\mu_k' = \mu_k - \Delta_{k_L+1}/2.$$ 
(b) holds since $\mu'_{k_L+1} \leq \tau \leq \mu_{k_L+1}$. Since $\Delta_{k_L+1} \leq \Delta_{k_L+2} \leq 3\Dmin/2$, we know that (c) and (d) hold. Finally note that 
\begin{align*}
    \mu_{k_L+1} - \mu_{k_L} \geq \mu_{k_L+1}' - \mu_{k_L}' = 
    3\Delta_{k_L+1}/8 + \Delta_k\\
    \geq \mu_{k_L+2}' - \mu_{k_L+1}' +  5\Delta_{k_L+1}/8 =
 \mu_{k_L+2}' - \mu_{k_L+1}' \geq \mu_{k_L+2} - \mu_{k_L+1},
\end{align*}
since $\mu_{k_L+2}' - \mu_{k_L+1}' \leq \Dmin/2$ - since $\Delta_{k_L+1} \leq \Delta_{k_L+2} \leq 3\Dmin/2$ - so that $\Delta_k \geq \Dmin \geq  \mu_{k_L+2}' - \mu_{k_L+1}' +  \Delta_{k_L+1}/4$. So (a) holds since for any $k \not\in \{k_L+1, k_L+2\}$, we have $\mu_k - \mu_{k-1} = \mu_k' - \mu_{k-1}' $.

\end{proof}

\noindent
\begin{proof}[Proof of Theorem \ref{thm:depcon_down}]\label{proof:depcon_down}
Consider $\bD \in \DeltaBc$ associated with the vector of means $(\mu_k)_{k\in [K]}$. We define $\unu$ as the Gaussian bandit problem with these means, that is, $\nu_k = \cN(\mu_k,\sigma^2)$ for all $k\in[K]$. Thanks to Lemma~\ref{lem:alternative_lb_ccv} there exists a vector of means $(\mu_k')_{k\in [K]}$  that verifies the conditions of Lemma~\ref{lem:alternative_lb_ccv}. We denote by $\unu'$ the Gaussian bandit problem such that $\nu_k' = \cN(\mu_k,\sigma^2)$ for all $k\in[K]$. Thanks to (a) and (d) we know that $\unu'\in\Bc$. Thanks to (b) there exists $i\in[K]$ such that, for example, $\mu_i > \tau$ and $\mu'_a <\tau$. In particular we can lower bound the error by the probability to make a mistake in the prediction of the label of arm~$i$
\[
e_T^{\unu} \geq \PP_{\unu}(\hQ_i = -1) \qquad e_T^{\unu'} \geq \PP_{\unu}(\hQ_i = 1)\,.
\]

We then conclude as in the proof of Theorem~\ref{thm:depmon_down}. We can assume that $\PP_{\unu}(\hQ_i = -1) \leq 1/2$ otherwise the bound is trivially true. Thanks to (c), the chain rule, the contraction of the Kullback-Leibler divergence and \eqref{eq:fano}, it holds \begin{align*}
    T\frac{9\Dm^2}{2\sigma^2} &\geq \KL(\PP_{\unu}^{I_T}, \PP_{\unu'}^{I_T})\\
    & \geq \kl\big(\PP_{\unu}(\hQ_i = 1),\PP_{\unu'}(\hQ_i = 1)\big)\\
    &\geq \PP_{\unu}(\hQ_i = 1) \log\!\!\left(\frac{1}{\PP_{\unu'}(\hQ_i = 1)}\right) -\log(2)\,,
\end{align*}
where we denote by $\PP_{\unu}^{I_T}$ the probability distribution of the history $I_T$ under the bandit problem $\unu$. Thus, using that $\PP_{\unu}(\hQ_i = 1) = 1-\PP_{\unu}(\hQ_i = -1) \geq 1/2$ we obtain 
\[
\PP_{\unu'}(\hQ_i = 1) \geq \frac{1}{4}\exp\!\left(-9\frac{ T\Dm^2}{\sigma^2}\right)\,.
\]
Which allows us to conclude that 
\[
\max(e_T^{\unu^+},e_T^{\unu^-}) \geq \frac{1}{4}\exp\!\left(-9\frac{T\Dm^2}{\sigma^2}\right)\,.
\]

% \[\xi := \left\{\left|\sum_{t=1}^T X_{t} - \sum_{t=1}^T\mu_{k_t}\right| \leq \sqrt{\frac{\sigma^2}{T}}\right\}\]

% which via a hoeffding martingale inequality occurs with probability greater than $\frac{3}{4}$. Now since by Lemma~\ref{}\al{tolink}, we know that $|\mu_k - \mu_k'| \leq 3\Dm$ for any $k$ we have \al{from here on to correct - constants, signs, etc... global results should be kinda fine}
% \begin{align*}
% \bP_{\unu^-}\left(\ind_{\cA}\right) &=  \bP_{\nu^{+}}\left(\ind_{\cA} \exp \left(-\frac{1}{2\sigma^2}\sum_{t=1}^T \left[( 2(X_t - \mu_{k_t}) (\mu_{k_t} - \mu_{k_t}')  + (\mu_{k_t} - \mu_{k_t}')^2\right]\right)\right)\\
% &\geq  \bP_{\nu^{+}}\left(\ind_{\cA\cap\xi} \exp \left(- 9\frac{T\Dm^2}{\sigma^2} - 3\frac{\Dm\sqrt{T}}{2\sigma}\right)\right)\\
% &\geq\bP_{\nu^{+}}\left(\ind_{\cA\cap \xi} \exp \left(- T\Dm^2\right)\right)\;,
% \end{align*}

% Consider the two cases $\mathbb{P}_{\nu^{+}}\left(\mathcal{A}\right) < \frac{1}{2}$ and  $\mathbb{P}_{\nu^{+}}\left(\mathcal{A}\right) > \frac{1}{2}$. In the first case we immediately have that, \al{Here notations weird}

% $$\tilde R_{T}^{\nu^{+\Delta},\pi} \geq \frac{1}{2}\;.$$
%  In the later case note that $$\mathbb{P}_{\nu^{+\Delta}}\left(\ind_{\mathcal{A}_i\cap\xi}\right) > \frac{1}{4}\;,$$
 
%  Thus 
%  $$\tilde R_{T}^{\nu,\pi} \geq \frac{1}{4}\exp \left(- T\Dm^2 \right)\;.$$
% With these two cases in mind,
% the proof is complete.
\end{proof}

\paragraph{Proof of Theorem \ref{thm:depcon_up}.}\label{proof:depcon_up}
We assume in the proof, without loss of generality, that \[
\Dmin \geq \cconmin \sqrt{\frac{\sigma^2\log(K)}{T}}
\]
    with $\cconmin = 8064$. Indeed, otherwise, the bound of Theorem~\ref{thm:depcon_up} is trivially true. 

The proof of Theorem \ref{thm:depcon_up} is structured in the following manner. In our \emph{original} binary tree we assume there is at least one arm above threshold, the contrary case is dealt with separately, see Lemma \ref{lem:armsbellow}. We wish to show that with high probability the \gradexplore algorithm will add sufficient arms above threshold to the list $S_{T_1}$ such that when we take it's median we are guaranteed to output an arm above threshold. At time $t$ we say our algorithm makes a favourable decision if all sample means are well concentrated - that it with $\Dm$ of their true mean. It is important to note that for arms below threshold this also implies the estimated gradients are close to their true values. On such a favourable decision we show that the explore algorithm will make a step towards the subtree of nodes containing an arm above threshold, or remain inside if it is already in it. We also show that upon encountering an arm above threshold, on a good decision said arm is always added to $S_{T_1}$. Therefore if overall we can make sufficient proportion of favourable events we are guaranteed to have a sufficient number of arms above threshold in $S_{T_1}$. We then show that this favorable event holds with high probability. Once we have identified an arm above threshold the problem is essentially split into two monotone problems - see Remark~\ref{rem:relaxing_monotone}, where the point the arms cross threshold on either side can be found by applying the \dexplore and \explore algorithms in opposite directions.

\paragraph{Step 1: Initial definitions and lemmas} 
We thus assume first that there is an arm $k*$ such that $\mu_{k^*} >\tau$.
\begin{definition}%\label{def:subtree}
We define the subtree $ST(v)$ of a node $v$ recursively as follows: $v \in ST(v)$ and
$$\forall \; q \in ST(v), \; L(q), R(q) \in ST(v)\;.$$
\end{definition} 

\begin{definition}\label{def:consecutive} 
A consecutive tree $U$ with root $u_\root$ is a set of nodes such that $u_{\root} \in U$ and 

\[\forall v \in U : v \neq u_{\root},\, P(v) \in U. \]
with the additional condition, 
\[\root \in U \Rightarrow u_{\root} = \root\]
where $\root$ is the root of the entire binary tree.

\end{definition}
We define $Z$, the set of good nodes with at least an arm with a mean above the threshold,
\begin{equation*}%\label{eq:z1}
Z:= \{ v : \exists j \in \{l,m,r\} : \mu_{v(j)} >\tau \}\;.
\end{equation*} 
At a given time $t$ note the following property of $Z$ and $v_t$, 
\begin{equation}\label{eq:zk*}
    ST(v_t) \cap Z \neq \emptyset \Leftrightarrow k^* \in [v_t(l),v_t(r)]\,.
\end{equation}

\begin{proposition}
\label{prop:Z}
$Z$ is a consecutive tree with root $z_{\mathrm{root}}$ the unique element $v \in Z$ such that $P(v) \notin Z$ if there exists at least one, otherwise $z_{\mathrm{root}} = \root$.
\end{proposition} 
\begin{proof}
First, if for all $v\in Z$ we have $P(v) \in Z$ then $\root \in Z$ and $Z$ is a consecutive tree with root $z_{\mathrm{root}}=\root$. Otherwise,
consider $v \in Z$, such that $P(v) \notin Z$, there is at least one such node. We first prove that $v$ is unique. As $v \in Z$ we know that 

\begin{equation}\label{eq:Z}
\exists j \in \{l,m,r\} :  \mu_{v(j)} > \tau \;.
\end{equation} 
Now since $v(l), v(r) \in P(v)$ and $P(v) \notin Z$, it follows that, thanks to \eqref{eq:Z},

$$ \forall k \in \{l,r\} : \mu_{v(k)} < \tau  \;.$$
For node $q \neq v$ satisfying the same properties, assume that $v(m) < q(m)$ without loss of generality. With this assumption we have,

\[v(r) \leq v(m) \leq q(l) \leq q(m)\;,\]
however this then implies $\mu_{q(l)} > \tau$ a contradiction. Hence $v = q$, and thus $v$ is unique which implies $\forall q \in Z:\ q \neq v,\, P(q) \in Z$.
\end{proof}

At time $t$ we define $w_t$ as the node of maximum depth whose sub tree contains both $v_t$ and  $Z$. Formally, for $t \leq T_1$,
\begin{equation}\label{eq:conw}
w_t := \argmax_{\{ST(w) \cap Z \neq \emptyset \; \& \;v_t \in ST(w)\}} |w|\,.
\end{equation} 

\begin{lemma}\label{lem:d}
The node $w_t$ is unique.

\end{lemma}
\begin{proof}
At time $t$ consider, a node $q_t$ which also satisfies \eqref{eq:conw}. As $v_t \in ST(w_t)$ and $v_t \in ST(q_t)$ we can assume without loss of generality $q_t \in ST(w_t)$ with $|q_t| \geq |w_t|$. This then implies, from \eqref{eq:conw}, that $|q_t| = |w_t|$ and as $q_t \in ST(w_t)$, we have $q_t =  w_t$.
\end{proof}

For $t\leq T_1$ we define $D_t$ as the distance from $v_t$ to $Z$, it is taken as the length of the path running from $v_t$ up to $w_t$ and then down to an good node in $Z$. Formally, we have
\[D_t := \left|v_t\right| - \left|w_t\right| + \left(|z_{\root}|-|w_t| \right)^{+}. \]
Note the following properties of $D_t$ and $w_t$,
\begin{gather*}
    ST(v_t) \cap Z \neq \emptyset \Rightarrow v_t = w_t \;,%\label{eq:ST} \\
    D_t = 0 \Rightarrow v_t = w_t \; \text{and} \; w_t,v_t \in Z \;. %\label{eq:D0}
\end{gather*}
Define at time $t$ the counter $G_t$, tracking the number of good arms in $S_t$, 
\begin{equation}\label{eq:gdef}
    G_t := \Big|\big\{k \in S_t:\ \mu_k >\tau \big\}\Big| \;.
\end{equation}
At time $t$ we define the following favorable event where the sampled arms at time $t$ a well concentrated around their means, 
\[\xi_t := \{\forall j\in \{l,l+1,m,m+1,r,r+1\}, \left|\hmu_{j,t} - \mu_{v_{t}(j)}\right| \leq \Dmin\}.\]

\paragraph{Step 2: Actions of the algorithm on all iterations}\label{par:2}

After any execution of algorithm \gradexplore note the following,
\begin{itemize} 
\item for $t \leq T_1$,\; $v_t$ and $v_{t+1}$ are separated by at most one edge, i.e. 
\begin{equation}\label{eq:1edge_concave}
v_{t+1} \in \{L(v_t), R(v_t), P(v_t) \}\,,
\end{equation} 
\item for $t \leq T_1$, 
\begin{equation} \label{eq:S}
|S_t| \leq |S_{t+1}| \leq |S_t| +1 \,. 
\end{equation} 
\end{itemize}

\begin{lemma}\label{lem:allit} 
On execution of algorithm \gradexplore for all $t \leq T_1$ we have the following,
\begin{gather}
    D_{t+1} \leq D_t + 1 \label{eq:dxi}, \\
    G_{t+1} \geq G_t\label{eq:bg} \,.
\end{gather}

\end{lemma}
\begin{proof}
As the algorithm moves at most 1 step per iteration, see~\eqref{eq:1edge_concave}, for $t \leq T_1$,  it holds
$$\left| \left|v_{t}\right| - \left|w_{t}\right| \right| \geq \left| \left|v_{t+1}\right| - \left|w_{t} \right| \right| -1\;.$$
Noting that,
\begin{align*}
D_t &= \left| \left|v_{t}\right| - \left|w_{t}\right| \right| + \left(|z_{\root}|-|w_t| \right)^{+} \\
& \geq \left| \left|v_{t+1}\right| - \left|w_{t}\right| \right| + \left(|z_{\root}|-|w_t|\right)^{+} -1\\
& \geq \left| \left|v_{t+1}\right| - \left|w_{t+1} \right| \right| + \left( |z_{\root}| - |w_{t+1}|\right)^{+} -1\\
&= D_{t+1} - 1\;,
\end{align*}
where the third line comes from the definition of $w_{t+1}$, see~\eqref{eq:w}, we obtain $D_{t+1} \leq D_t + 1$.
By~\eqref{eq:S} we have, for $t \leq T_1$, 
$$|S_t| \leq |S_{t+1}| \leq |S_{t}| + 1\,, $$
hence $G_{t+1} \geq G_t$.
\end{proof} 

\paragraph{Step 3: Actions of the algorithm on $\xi_t$}\label{par:3}
We first state several properties relating to the event $\xi_t$.
Firstly for all $t$ we have that under event $\xi_t$,
\begin{equation}\label{eq:ximean}
    \forall k \in \{l,m,r\},\;
    \sign(\hmu_{k,t} - \tau) = \sign(\mu_k - \tau)\;.
\end{equation}
Since there is at least an arm above the threshold, due to the concave property, note the following,
\begin{equation}\label{eq:deltamin}
    \forall k\in[K]:\ \mu_k < \tau,\; |\mu_k - \mu_{k+1}| \geq 2\Dmin\;,
\end{equation}
thus from \eqref{eq:deltamin} for all $t$ under event $\xi_t$, we have that,
\begin{equation}\label{eq:nab}
    \forall j\in \{l,m,r\} : \mu_{v_t(j)} < \tau, \; \sign(\hnabla_{j,t}) = \sign(\nabla_{v_t(j)})\,.
\end{equation}

\begin{lemma}\label{lem:dxic}
On execution of algorithm \gradexplore  for all $t \leq T_1$, on $\xi_t$, we have the following,
\begin{gather}
    D_{t+1} \leq \max(D_t - 1,0)\;,     \label{eq:dxic}\\
    G_{t+1} \geq G_t + \ind_{\{D_t = 0\}}\,. \label{eq:gxic} 
\end{gather}
\end{lemma}

\begin{proof}
We first prove~\eqref{eq:gxic}. 
% Note that if the arm $v_t(j)$ is added in $S_{t+1}$ then $\hmu_{j,t}>\tau$. Thus, on $\xi_t$, we have also $\mu_{v_t(j)} >\tau$, see \eqref{eq:ximean}, hence $G_{t+1} \geq G_t +1$. It remains to prove that, when $D_t=0$, an arm is effectively added in $S_{t+1}$.
If $D_t=0$ then we know $v_t\in Z$. If $v_t \in Z$ then under $\xi_t$ there exists $j\in\{l,m,r\}$ such that $\hmu_{j,t} > \tau$, see \eqref{eq:ximean}, and arm is added to $S_{t+1}$, thus $G_{t+1} \geq G_t + \ind_{\{D_t = 0\}}$.

% Note that on the favorable event $\xi_t$, we have $\forall k \in \{l,m,r\}$,
% \begin{gather}\label{eq:1}
% \mu_{v_t(k)} \geq \tau  \Rightarrow \hat\mu_{k,t} \geq \tau\,, \\
% \mu_{v_t(k)} \leq \tau  \Rightarrow \hat\mu_{k,t} \leq \tau \, \label{eq:2}.
% \end{gather} 
We now prove~\eqref{eq:dxic}. We consider the following three cases:
\begin{itemize} 

\item If $Z \cap ST(v_t) = \emptyset$. First of all we have that $\forall j \in\{l,m,r\} :\mu_{v_t(j)} \leq 
\tau$. Therefore from \eqref{eq:ximean} the algorithm will not add an arm to $S_t$. Now, we have that $k^* \notin [v_t(l),v_t(r)]$, see \eqref{eq:zk*}, therefore via the concave property $\nabla_{v_t(l)} < 0$ or $\nabla_{v_t(r)} > 0$. Via \eqref{eq:nab} this implies that $\hnabla_{v_t(l)} < 0$ or $\hnabla_{v_t(r)} > 0$ respectively. Thus by action of the algorithm $v_{t+1} = P(v_t)$. Since in this case we are getting closer to the set of good nodes by going up in the tree we know that $w_t = w_{t+1}$. Thus thanks to Lemma~\ref{lem:d}, under $\xi_t$,
$$  D_{t+1} = \left|v_{t+1}\right| - \left|w_{t+1}\right| + \left(|z_{\root}|-|w_{t+1}| \right)^{+} = \left|v_{t}\right| -1 - \left|w_{t}\right| + \left(|z_{\root}|-|w_{t}| \right)^{+} = D_t - 1 \,.$$

\item If $k^* \in ST(v_t)$ and $v_t \notin Z$. First of all we have that $\forall j \in\{l,m,r\} :\mu_{v_t(j)} \leq 
\tau$. Therefore from \eqref{eq:ximean} the algorithm will not add an arm to $S_t$. Now note that in this case $v_t$ can not be a leaf and we just need to go down in the subtree of $v_t$ to find an good node, id est $w_t = v_t$. Since $v_t\notin Z$, without loss of generality, we can assume for example $\hnabla_{t,m}  > 0$. From~\eqref{eq:nab}, under $\xi_t$, we then have that $\nabla_{v_t(m)} > 0$ which implies $k^* \in [v_t(l),v_t(m)]$.
 Hence algorithm \gradexplore goes to the correct subtree, $v_{t+1} = L(v_t)$. In particular we also have for this node 
$$ k^* \in \left[v_{t+1}(l),\,v_{t}(m)\right] \,,$$
therefore it holds again $w_{t+1} = v_{t+1}$. Thus combining the previous remarks we obtain thanks to Lemma~\ref{lem:d}, under $\xi_t$,
$$ D_{t+1} = \left(|w_{t+1}| - |z_{\root}| \right)^{+} = \left(|w_t| - |z_{\root}| \right)^{+}-1 =  D_t - 1 \;.$$

\item If $k^* \in ST(v_t)$ and $v_t \in Z$. In this case there exists $j \in \{l,m,r\}$ such that  $\mu_{v_t(j)} > 
\tau$. From \ref{eq:ximean} we have for said $j$ that, $\hmu_{j,t} > \tau$. Hence the algorithm will not move giving $v_t = v_{t+1}$ thus $D_t =D_{t+1} = 0$.
\end{itemize} 
\end{proof}

\paragraph{Step 4: Lower bound on $G_{T_1+1}$}
We denote by $\bxi_t$ the complement of $\xi_t$.
\begin{lemma}\label{lem:lb_on_Gt} For any execution of algorithm \gradexplore,
\[
G_{T_1+1} \geq \frac{3}{4} T_1 - 2 \sum_{t=1}^{T_1}\ind_{\bxi_t}\,.
\]
\end{lemma}
\begin{proof}
Combining~\eqref{eq:dxic} and~\eqref{eq:dxi} from Lemma~\ref{lem:allit} and Lemma~\ref{lem:dxic} respectively we have
\begin{align*}
D_{t+1} &\leq D_t+\ind_{\bxi_t}-\ind_{\xi_t}\ind_{\{D_t>0\}}\\
&= D_t+2\ind_{\bxi_t}-1+\ind_{\xi_t}\ind_{\{D_t=0\}}\,.
\end{align*}
Using this inequality with~\eqref{eq:gxic} we obtain 
\begin{align*}
G_{T_1+1} &= \sum_{t=1}^{T_1} G_{t+1}-G_{t} \\
&\geq \sum_{t=1}^{T_1}\ind_{\xi_t}\ind_{\{D_t=0\}}\\
&\geq \sum_{t=1}^{T_1} \big(D_{t+1}-D_t -2\ind_{\bxi_t}+1\big)\\
&\geq T_1-D_1-2\sum_{t=1}^{T_1}\ind_{\bar{\xi}_{t,}}\\
&\geq \frac{3}{4} T_1 -2\sum_{t=1}^{T_1}\ind_{\bxi_{t,}}\,,
\end{align*}
where we used in the last inequality the fact that $D_1\leq \log_2(K)$ and that $\log_2(K) \leq T_1/4$ by definition of $T_1$ .
\end{proof}

\begin{lemma}\label{lem:conchern} 
Upon execution of algorithm \gradexplore with budget $\frac{T}{3}$ we have that,
\[
\PP\left( \sum_{t=1}^{T_1} \ind_{\bar{\xi}_t} \leq \frac{T_1}{8}\right) \leq  \exp\left(-\ccon  \frac{T \bD_{\min}^2}{\sigma^2}+\ccon'\log(K)\right) \,.
\]
where $\ccon = \frac{1}{576}$ and $\ccon' = 12$. 
\end{lemma}
\begin{proof}
 The proof follows as in the proof of Lemma \ref{lem:chern}, with altered constants.

\end{proof}

\begin{lemma}\label{lem:gradex}
Under the assumption $\exists k: \mu_k > \tau$, upon execution of algorithm \gradexplore with output $\hk$ we have that $\mu_{\hk} \geq \tau$ with probability greater than $$1-\exp\left(-\ccon  \frac{T \bD_{\min}^2}{\sigma^2}+\ccon'\log(K)\right)\;.$$
\end{lemma}
\begin{proof} 
By combination of Lemmas \ref{lem:lb_on_Gt} and \ref{lem:conchern} we have that $G_{T_1+1} \geq \frac{1}{2}T_1$ with probability greater than $1 -\exp\left(-\ccon \frac{T \bD_{\min}^2}{\sigma^2} \ccon'\log(K)\right)$. As $|S_{T_1+1}| \leq T_1$ and as the arms $G_{T_1 + 1}$ form a segment (they are all above threshold) by taking the median of $S_{T_1 + 1}$ under the circumstance $G_{T_1+1} \geq \frac{1}{2} T_1$ we have that the output of \gradexplore $\hk$ is such that $\mu_{\hk} > 
\tau$. This then gives the result.
\end{proof}

With the following lemma we deal with the special case where all arms are below threshold before finally completing the proof of Theorem \ref{thm:depcon_up}.

\begin{lemma}\label{lem:armsbellow}  
Under the assumption $\forall k \in [K]: \mu_k < \tau$, upon execution of algorithm \gradexplore with output $\hk$ we have that $\forall k\in [K], \hQ_k = -1$  with probability greater than 
$$1-\exp\left(-\ccon  \frac{T \bD_{\min}^2}{\sigma^2}+\ccon'\log(K)\right)\;.$$
where $\ccon = \frac{1}{576}$ and $\ccon' = 12$. 
\end{lemma}
\begin{proof}
Under the assumption $\forall k \in [K]: \mu_k < \tau$, for all $t < T_1$, we have that under the event $\xi_t$, $S_{t+1} = S_t$, see \eqref{eq:ximean}. Therefore the following holds, 

\[|S_{T_1}| \leq \sum_{t=1}^{T_1} \ind_{\bar{\xi}_t}\;.\]
The proof now follows from direct application of Lemma \ref{lem:conchern}. 
\end{proof} 

We are now ready to prove Theorem \ref{thm:depcon_up}.\medskip
\newline
\begin{proof}[Proof of Theorem \ref{thm:depcon_up}]
In the case where $\mu_k < \tau,\; \forall k \in [K]$ Lemma \ref{lem:armsbellow} immediately gives the result. Therefore we consider the case in which $\exists k \in [K] : \mu_k > \tau$. Under this assumption the algorithm \gradexplore will return an arm $\hk : \mu_{\hk} > \tau$ with probability greater than 

$$1-\exp\left(-\frac{1}{576}  \frac{T \bD_{\min}^2}{\sigma^2}+12\log(K)\right)\;,$$

see Lemma \ref{lem:gradex}. In this case we have the sets of arms $[1,\hk]$, $[\hk,K]$ which satisfy the assumption described in Remark \ref{rem:relaxing_monotone}. Therefore via Theorem \ref{thm:depmon_up} and a union bound we have that with probability greater than 
\[1 - 2 \exp\!\left(-\frac{1}{48} \frac{T\Delta^2}{\sigma^2}+12 \log(K) \right)\]
we will correctly classify arms on both these sets. With an additional union bound we achieve the result.  
\end{proof}
\section{Experiments}
\label{app:experiments}
% !TEX root = ../main.tex

We conduct some preliminary experiments to test the performance of both \explore and \CTB to illustrate our theoretical understanding. As a bench mark we will use both a \uniform algorithm and also a naive binary search - that is without back tracking, that we will term \naive, for an exact description of both see Appendix. Note that \naive essentially behaves as a uniform sampling algorithm on a bandit problem with $\log(K)$ arms. As our theoretical bounds are likely far to loose in terms of constants we also include a parameter tuned version of the \explore where we tune the constants in the definition of $T_1$
and $T_2$, see Equation~\eqref{eq:t1t2}.

We would expect the \naive algorithm to have an upper bound of the order $\exp\left(-\frac{T\Dm^2}{\log(K)}\right)$. This is sub-optimal compared to \explore which removes the $\log(K)$, see Theorem~\ref{thm:depmon_up}. However, \explore must divide it's budget across several arms at each round, while \naive algorithm samples only one. This may out weigh the benefit of backtracking when $K$ is not very large. 

In our experiments we consider two thresholding bandit problems. In Setting~1 the gap of one arm is set to $\Delta$, with the remaining gaps very large - i.e. 100,  In Setting~2 all gaps are set to $\Delta$, for the \CTB we modify this to a concave setting where all arms are Delta apart. The former problem should more favour \explore as it can quickly traverse the binary tree and expend most of it's budget on the leaf in question. 

In Figure \ref{fig:E1} we consider consider the expected error in Setting 1 as a function of the gap $\Delta$ and as a function of the number of arms $K$.The effect of varying $\Delta$ follows our intuition. Firstly all algorithms show an increased performance for greater $\Delta$, this should be completely expected. Secondly, in Setting~1 the \explore algorithm decrease in probability of error faster than Naive and much faster than \uniform. This is also unsurprising as in this setting the \uniform, and to a lesser extent \naive, algorithms are forced to waste an unnecessary amount of their budget on arms far from threshold. In the case of varied K, on the right, \explore appears to outperform \naive, showing no obvious dependency on $K$ past a certain point, however there is considerable noise.

\begin{figure}[H]\label{fig:diffdelt}
\includegraphics[width=.45\linewidth]{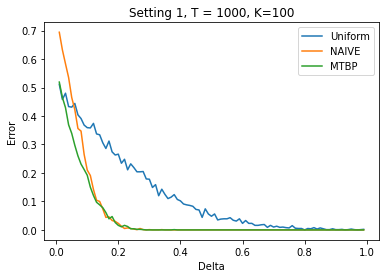}
\includegraphics[width=.45\linewidth]{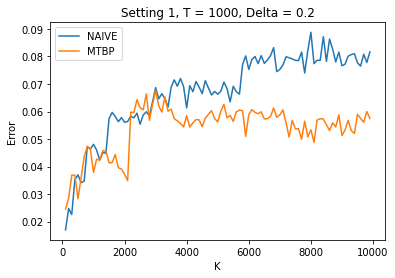}\\
\caption{On the right: expected error as a function of the number of arms $K \in \left(100\times i\right)_{i\in[100]}$ with $T = 1000$ and $\Delta = 0.2$ in Setting~1 averaged over 10000 Monte Carlo simulations. On the left: expected error as a function of the gap  $\Delta \in \left(0.01\times i\right)_{i\in[100]}$ with $T = 1000$ and $K = 100$ in Setting~1 averaged over 1000 Monte Carlo simulations.}
\label{fig:E1}
\end{figure}

In figure \ref{fig:E2} we consider consider the expected error in Setting 2 as a function of the gap $\Delta$ and as a function of the number of arms $K$. In both cases Naive out performs both \explore and it's tuned version, vastly so for larger $K$. It would appear that here dividing our budget cancels out any gains one receives from reducing dependency on $\log(K)$. It is unfortunate that we were unable to find heuristic evidence of a lack of dependency on $\log K$, although this was perhaps expected. Based on our results, see Theorem~\ref{thm:depmon_up}, to remove such a dependency one would need $T\Delta^2 >> \log(K)$. This would lead to extremely low probabilities of error which are near impossible to detect accurately without huge numbers of Monte Carlo simulations, unfortunately beyond the scope of this paper.

\begin{figure}[!h]
\includegraphics[width=.45\linewidth]{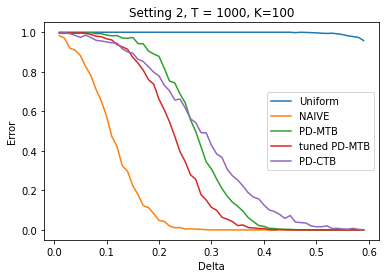}
\includegraphics[width=.45\linewidth]{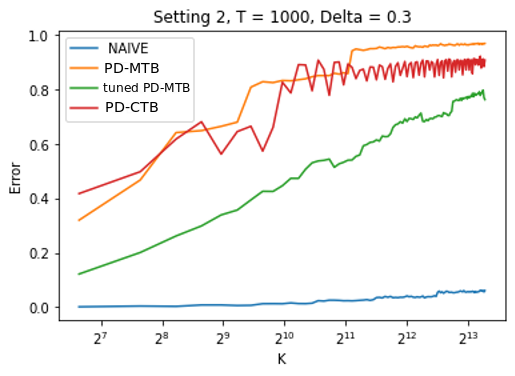}\\
\caption{On the right: Expected error as a function of the number of arms $K \in \left(100\times i\right)_{i\in[100]}$ with $T= 1000$, $\Delta = 0.3$, in Setting~2, plotted on a log scale averaged over 10000 Monte Carlo simulations. On the left: expected error as a function of the gap  $\Delta \in \left(0.01\times i\right)_{i\in[60]}$ with $K = 100$, $T=1000$, in Setting~2, averaged over 1000 Monte Carlo simulations}\label{fig:E2}
\end{figure}

\begin{algorithm}[H]
\caption{\naive}
\label{alg:naive} 
\begin{algorithmic}
\STATE {\bfseries Initialization:} $v_1=\root$
\FOR{$t=1:T_1$}
\STATE sample $\lfloor\frac{T}{\log(K)}\rfloor$ times each arm in $v_t(m)$\\
\IF{$\hat{\mu}_{m,t} \leq \tau $}
\STATE$v_{t+1} = R(v_t)$
\ELSE
\STATE $v_{t+1} = L(v_t)$
\ENDIF
\ENDFOR
\STATE Set $\hk = v_{T_1 + 1}(r)$\\
\RETURN $(\hk, \hQ) :\quad  \hQ_k=  2\ind_{\{k \geq \hk \}}-1$
\end{algorithmic}
\end{algorithm}

\begin{algorithm}[H]
\caption{Uniform}\label{alg:unif} 
\begin{algorithmic}
\FOR{$k=1:K$}
    \STATE Sample arm $k$ a total of 
    \STATE $\lfloor\frac{T}{K}\rfloor$ times.\\
    \STATE Compute $\hmu_k$ the sample mean of arm $k$. 
\ENDFOR
\RETURN\vspace{-5mm} \begin{equation*}\hQ :\quad    \hQ_k=
    \begin{cases}
      -1 & \text{if}\ \hmu_k < \tau \\
      1 & \text{if}\ \hmu_k \geq \tau
    \end{cases}
  \end{equation*}
\end{algorithmic}

\end{algorithm}

\end{document}